\documentclass[twoside]{article}

\usepackage[preprint]{aistats2026}

\usepackage[round]{natbib}
\usepackage{subfiles}

\usepackage{mathtools}
\mathtoolsset{showonlyrefs}
\usepackage{amsmath,amssymb,amsthm}
\usepackage{bm}
\usepackage{dsfont}
\usepackage{enumitem}
\usepackage[final]{showlabels} %
\usepackage{xcolor}         %
\usepackage{booktabs} %
\usepackage{hyperref}
\usepackage[mathscr]{eucal}    %

\usepackage{soul}

\newtheorem{theorem}{Theorem}
\newtheorem{corollary}{Corollary}
\newtheorem{assumption}{Assumption}

\newtheorem{remark}{Remark}
\newtheorem{lemma}{Lemma}

\newtheorem{theoremA}{Theorem}[section]
\newtheorem{lemmaA}{Lemma}[section]

\definecolor{Brown}{rgb}{0.64,0.16,0.16}
\definecolor{OliveGreen}{rgb}{0.1,0.4,0.1}

\renewcommand{\b}{\mathbf}    %
\newcommand{\R}{\mathbb{R}}   %
\renewcommand{\P}{\mathbb{P}} %
\newcommand{\Q}{\mathbb{Q}} %
\newcommand{\E}{\mathbb{E}} %
\newcommand{\Ds}[1]{D_{\P_0}^2\!\!\left(#1\right)} %
\newcommand{\D}[1]{D_{\P_0}\!\left(#1\right)} %

\newcommand{\Dq}[1]{D_{\Q_0}\!\left(#1\right)}
\newcommand{\Dqs}[1]{D_{\Q_0}^2\!\!\left(#1\right)}

\newcommand*{\T}{^{\mkern-1.5mu\mathsf{T}}} %
\renewcommand{\d}{\mathrm{d}} %
\newcommand{\tb}{\textbf}    %
\newcommand{\imag}{i} %
\def\w{\bm \omega}
\renewcommand{\Re}{\text{Re}}
\renewcommand{\Im}{\text{Im}}
\newcommand{\co}{\overline}
\newcommand{\bigfip}[2]{\Big\langle{#1}\Big\rangle_{#2}} %

\newcommand{\fip}[2]{\left\langle{#1}\right\rangle_{#2}} %
\newcommand{\norm}[2]{\left\|{#1}\right\|_{#2}} %
\newcommand{\trace}{\operatorname{tr}}
\newcommand{\KL}{\operatorname{KL}}
\newcommand{\supp}{\operatorname{supp}}

\newcommand{\fm}[1]{K_0(\cdot,#1)}

\renewcommand{\Q}{Q}
\renewcommand{\P}{P}
\renewcommand{\Ds}[1]{\KSD^2(P_0,#1)} %
\renewcommand{\D}[1]{\KSD(P_0,#1)} %

\renewcommand{\Dq}[1]{\KSD(G_0,#1)}
\renewcommand{\Dqs}[1]{\KSD^2(G_0,#1)}

\newcommand{\C}{\mathbb{C}} %
\newcommand{\ip}[2]{\langle #1, #2 \rangle}
\renewcommand{\H}{\mathcal{H}} %

\newcommand{\Zp}{\mathbb{N}_{>0}}
\newcommand{\N}{\mathbb{N}_{0}}

\newcommand{\mP}{\mathcal{P}}

\renewcommand{\d}{\mathrm{d}}
\newcommand{\p}{\partial}
\renewcommand{\T}{\top}
\renewcommand{\H}{\mathscr{H}}
\renewcommand{\b}{\mathbf}
\renewcommand{\t}{\text}
\renewcommand{\O}{\mathcal{O}}
\newcommand{\mA}{\mathcal{A}}
 
\newcommand{\X}{\mathcal{X}}
\newcommand{\Y}{\mathcal{Y}}

\DeclareMathOperator{\Sp}{Span}
\DeclareMathOperator{\KSD}{KSD}

\setenumerate{labelindent=0em,leftmargin=2.1em,topsep=0pt,parsep=0pt,itemsep=1mm}
\setitemize{labelindent=0em,leftmargin=1.3em,topsep=0pt,parsep=0pt,itemsep=1mm}

\begin{document}
\runningauthor{Jose Cribeiro-Ramallo, Agnideep Aich, Florian Kalinke, Ashit Baran Aich, Zoltán Szabó}

\twocolumn[

\aistatstitle{The Minimax Lower Bound of Kernel Stein Discrepancy Estimation}

\aistatsauthor{Jose Cribeiro-Ramallo \And Agnideep Aich }
\aistatsaddress{Karlsruhe Institute of Technology \And University of Louisiana at Lafayette }
\aistatsauthor{ Florian Kalinke  \And Ashit Baran Aich \And Zoltán Szabó}
\aistatsaddress{ Karlsruhe Institute of Technology \And Formerly of Presidency College \And London School of Economics}
]

\begin{abstract}
Kernel Stein discrepancies (KSDs) have emerged as a powerful tool for quantifying goodness-of-fit over the last decade, featuring numerous successful applications. To the best of our knowledge, all existing KSD estimators with known rate achieve $\sqrt n$-convergence. In this work, we present two complementary results (with different proof strategies), establishing that the minimax lower bound of KSD estimation is $n^{-1/2}$ and settling the optimality of these estimators. Our first result focuses on KSD estimation on $\mathbb R^d$ with the Langevin-Stein operator; our explicit constant for the Gaussian base kernel indicates that the difficulty of KSD estimation may increase exponentially with the dimensionality $d$. Our second result settles the minimax lower bound for KSD estimation on general domains.
\end{abstract}

\section{INTRODUCTION}
A fundamental problem in data science and statistics is quantifying the goodness-of-fit (GoF) between a known fixed target distribution and a sampling distribution (observed through samples only). A recent approach to tackle this challenging task employs the family of kernel Stein discrepancies (KSDs; \citealt{chwialkowski16kernel,liu16kernelized}), which combine a so-called Stein operator \citep{stein72bound,chen21stein,anastasiou23stein} with the flexibility and computational tractability of reproducing kernel Hilbert spaces (RKHSs; \citealt{aronszajn50theory}) associated to kernels. These kernel functions have been designed on a wide variety of domains, rendering KSDs broadly applicable.

KSDs rely on kernel mean embeddings \citep{berlinet04reproducing,smola07hilbert,gretton12kernel}, mapping probability measures to RKHSs without loss of information, under mild conditions. Considering the RKHS distance of two embedded probability distributions results in the maximum mean discrepancy (MMD), known to be equivalent \citep{sejdinovic13equivalence} to energy distance \citep{baringhaus04new,szekely04testing,szekely05new} (also known as $N$-distance; \citealt{zinger92characterization,klebanov05ndistance}), and to be a specific instance of integral probability metrics (IPM; \citealt{zolotarev83probability,muller97integral}). The key property guaranteeing that MMD is a metric is that the underlying kernel function is characteristic \citep{fukumizu07kernel,sriperumbudur10hilbert}.
When MMD is applied---with the product kernel---to the embeddings of a joint distribution and the product of its marginals, one obtains the Hilbert-Schmidt independence criterion (HSIC), originally designed for $M=2$ components \citep{gretton05measuring,gretton05kernel}, and later extended to $M\ge2$ components \citep{quadrianto09kernelized,sejdinovic13kernel,pfister18kernel}. HSIC is a valid independence measure for $M=2$ random variables if the kernel components are characteristic \citep{lyons13distance}; for $M>2$, $c_0$-universality of the kernel components suffices \citep{szabo18characteristic2}. HSIC can also be interpreted as the RKHS norm of the covariance operator; it is also equivalent \citep{sejdinovic13equivalence} to distance covariance \citep{szekely07measuring,szekely09brownian,lyons13distance}. Related mean embedding-based approaches constructed to measure the interaction of random variables include the kernel Lancaster and Streitberg interactions \citep{sejdinovic13kernel}, which, alongside MMD, HSIC ($M=2$), and maximum variance discrepancy \citep{makigusa24maximumvariance}, are specific cases of kernel cumulants \citep{bonnier23cumulants,liu23interaction}.

Similarly, KSD uses the mean embeddings of the target and the sampling distribution, where the underlying kernel is chosen such that the mean embedding of the target distribution vanishes. On Euclidean spaces, one attractive property of the classical Langevin-Stein KSD is that the resulting GoF measure is agnostic of the normalization constant of the sampling distribution, which can be challenging to compute in applications. This independence has led to its widespread use and its extension to other domains. Applications include model validation \citep{gorham17measuring,futami19bayesian,hodgkinson21reproducing,wang23stein}, learning variational models \citep{liu16stein,liu18stein,chen18stein,chen19stein,korba20svgd,korba21ksddescent}, testing  \citep{liu16kernelized,chwialkowski16kernel,schrab22ksdagg,kanagawa23steintest,hagrass24stein}, model comparison \citep{lim19kernel,kanagawa20kernel}, distribution compression \citep{li24debiased}, and model explainability \citep{sarvmaili25explaining}. KSD has also successfully been applied on discrete spaces \citep{yang18discrete}, Riemannian manifolds \citep{xu20stein,xu21interpretable,barp22riemannstein}, Hilbert spaces \citep{wynne24spectral}, point-processes \citep{yang19point}, and graph data \citep{xu21gofgraph}.

Despite their broad applicability, to the best our knowledge, convergence rates of KSD estimators have only been studied for V-statistic and Nyström-based estimators \citep{kalinke25nystromksd}. In fact, under a sub-Gaussian assumption, both estimators achieve $\sqrt n$-convergence on general domains.\footnote{As noted in the cited work, the $\sqrt n$-rate, while presented on $\R^d$, also holds on general domains.} Whether faster rates for KSD estimation are achievable is an open problem and the main focus of this work.

Answering this question requires obtaining minimax lower bounds and contrasting them with the existing upper bounds. Related
minimax lower bounds have been established for MMD \citep{tolstikhin16minimax}, the mean embedding \citep{tolstikhin17minimax}, covariance operators \citep{zhou19covestimators}, and HSIC \citep{kalinke24minimax}. While the proofs differ in all of the mentioned works, they (i) all assume the underlying kernel function to be bounded and (ii) rely on Le Cam's two point method (elaborated in Section~\ref{sec:proof-sketches}) to establish the minimax lower bounds.  Unfortunately, in the context of KSD, boundedness practically never holds, see, for example, \citep[Example~1]{kalinke25nystromksd} and \citep[Remark~2]{hagrass24stein}. Hence, existing results do not apply to the analysis of KSD estimation.
In this work, we address this gap by making the following
\tb{contributions}.
\begin{enumerate}[label=(\roman*)]
\item We establish the minimax lower bound $n^{-1/2}$ of KSD estimation on $\R^d$ with continuous bounded translation-invariant characteristic base kernels, with explicit constants for Gaussian base kernels.
\item Following a different proof strategy---by employing local perturbations---, we obtain the same lower bound for KSD estimation on general domains. Our imposed integrability conditions can be seen as a relaxation of the usual boundedness assumptions.
\end{enumerate}
The paper is structured as follows. Notations are introduced in Section~\ref{sec:notation}, followed  by recalling the notion of KSD (Section~\ref{sec:KSD}). Section~\ref{sec:KSD-estimators} is dedicated to existing KSD estimators with known convergence rates.  After recalling the minimax estimation framework (Section~\ref{sec:minimax}), our minimax results on KSD estimation are presented (Section~\ref{sec:results}) alongside their proof sketches (Section~\ref{sec:proof-sketches}). Detailed proofs are available in the appendix.

\section{NOTATIONS} \label{sec:notation}

In this section, we introduce our notations:
$\N$, $\Zp$, $\R$, $[n]$, $\{\{\cdot\}\}$, $(\cdot)^\T$, $\langle \cdot,\cdot\rangle_2$, $\|\cdot\|_2$, $\|\cdot\|_{\infty}$, $\b 1_d$, $\b x < \b y$, $\b A^{-}$, $\nabla$, $\supp$, $\bar{S}$, $\overline{z}$,  $\overline{\b z}$, $\Re(\cdot)$, $\Im(\cdot)$, $\b A^*$, $\langle \b x, \b y\rangle_{\C^d}$, $\|\b x\|_{\C^d}$, $|\b x|$, $\{\b e_j\}_{j=1}^d$, $\frac{\partial f}{\partial x_j}$, $\b x^{\bm \alpha}$, $D^{\bm \alpha}f$, $B(\H)$,  $\Sp$,  $\mathcal C^{s}(\R^d)$, $\mathcal C(\X)$, $\mathcal C_b(\X)$, $\mathcal{M}_1^+(\X)$, $\mathcal B(\X)$, $\lambda_d$, $\delta_{x}$, $\E_P[X]$, $P^n$, $\ll$, $\frac{\d Q}{\d P}$, $\KL(Q\|P)$, $M_{\bm \alpha}^P$,
$\mathcal N(\bm \mu, \bm \Sigma)$, $\psi_\P$, $\H_k$, $\H_k^d$, $\O$, $\Omega$, $\Theta$, $\O_P$.

The set of natural numbers is written as $\N=\{0,1,2,3,\dots\}$; the set of positive integers is denoted by $\Zp$; $\R$ stands for reals. Let $[n]\coloneq\{1,\ldots,n\}$ with $n\in \Zp$. We write $\{\{\cdot\}\}$ for a multiset. The transpose of a vector $\b v\in \R^d$ is written as $\b v^\T\in \R^{1\times d}$. The inner product of vectors $\b u=(u_j)_{j=1}^d, \b v = (v_j)_{j=1}^d \in \R^d$ is $\langle \b u, \b v \rangle_2 = \sum_{j=1}^{d} u_j v_j$.  The Euclidean norm of $\b x \in \R^d$ is $\|\b x\|_2 = \sqrt{\langle \b x, \b x \rangle_2}$; its supremum norm is $\|\b x\|_{\infty} = \max_{j\in [d]}|x_j|$. The $d$-dimensional vector of ones is denoted by $\b 1_d$. For vectors $\b x,\b y\in \R^d$, $\b x < \b y$ means that $x_j < y_j$ for all $j \in [d]$. The (Moore-Penrose) pseudo-inverse of a matrix $\b A \in \R^{d_1 \times d_2}$ is  $\b A^{-}\in \R^{d_2\times d_1}$. For a differentiable function $f: \R^d \to \R$, let $\nabla_{\b x} f(\b x)=\left(\frac{\p f}{\p x_j}(\b x)\right)_{j=1}^d\in \R^d$ ($\b x\in \R^d$). The support of a function $\varphi:\R^d \to \R$ is $\supp(\varphi)=\overline{\{\b x \in \R^d : \varphi(\b x) \neq 0\}}$, where $\bar{S}$ stands for the closure of the set $S$.
The conjugate of a complex number $z=a+ib\in \C$ is denoted by $\overline{z}=a-ib$ with $i=\sqrt{-1}$;  its real part is $\Re(z)=a$, its complex part is $\Im(z)=b$. On a vector $\b z=(z_j)_{j=1}^d\in \C^d$, conjugation, real part and complex part  act coordinate-wise: $\overline{\b z}=(\overline{z_j})_{j=1}^d$, $\Re(\b z) = (\Re(z_j))_{j=1}^d$, $\Im(\b z) = (\Im(z_j))_{j=1}^d$. The adjoint of a matrix $\b A \in \C^{d_1\times d_2}$ is written as $\b A^* \in \C^{d_2 \times d_1}$.  The inner product of vectors $\b x, \b y \in \C^d$ is $\langle \b x, \b y\rangle_{\C^d} = \b y^* \b x$; $\|\b x\|_{\C^d}\coloneq \sqrt{\langle \b x, \b x \rangle_{\C^d}} $ for $\b x \in \C^d$. Let $\b x = (x_i)_{i=1}^d \in \R^d$; we write $|\b x|\coloneq \sum_{j \in [d]}|x_{j}|$. Let $ \{\b e_j\}_{j=1}^d \subset \R^d$ be the canonical basis of $\R^d$. For $f:\R^d\rightarrow \C$, we define \(\frac{\p f}{\p x_j} (\b x) = \lim\limits_{h\rightarrow 0}\frac{f(\b x + h \b e_j) - f(\b x)}{h} = \lim\limits_{h\rightarrow 0}\frac{\Re \left(f(\b x + h \b e_j)\right) - \Re \left(f(\b x)\right)}{h} + i \lim\limits_{h\rightarrow 0}\frac{\Im \left(f(\b x + h \b e_j)\right) - \Im \left( f(\b x)\right)}{h} \) as the partial derivative of $f$ on $x_j$, and $\nabla_{\b x} f(\b x) = \left(\frac{\partial f}{\partial x_j}(\b x)\right)_{j=1}^d \in \C^d$ as the gradient of $f$ ($\b x \in \R^d$). Let $\bm \alpha = (\alpha_{j})_{j=1}^d \in \N^d$ and $\b x \in \R^d$. We write $\b x^{\bm \alpha}\coloneq \prod_{j=1}^d x_{j}^{\alpha_{j}}$ and $D^{\bm \alpha} f\coloneq \frac{\partial^{|\bm \alpha|} f}{\partial\b x^{\bm \alpha}}= \frac{\partial^{|\bm \alpha|} f}{\partial x_1^{\alpha_1}\cdots \partial x_d^{\alpha_d}}$.
Let $\H$ be a Hilbert space; $B(\H) =\{f \in \H\,:\, \left\|f\right\|_{\H}\le 1 \}$ denotes its unit ball centered at the origin. For a set $S$ in a vector space, $\Sp(S)$ stands for its linear hull. For $s\in \N$, the space of $s$-times continuously differentiable real-valued functions on $\R^d$ is denoted by $\mathcal C^{s}(\R^d)$. Let $\X$ be a topological space. The set of real-valued continuous functions on $\X$ is denoted by $\mathcal C(\X)$. The subspace of $\mathcal C(\X)$ consisting of bounded functions is denoted by $\mathcal C_b(\X)$.
The set of Borel probability measures on $\X$ is denoted by~$\mathcal M_{1}^+(\X)$, with $\mathcal B(\X)$ standing for the Borel sigma algebra on $\X$. Let $\lambda_d$ denote the Lebesgue measure on $\R^d$.
The Dirac measure centered at $x \in \X$ is denoted by $\delta_{x}$.
The expectation of a random variable $X \sim P \in \mathcal M_{1}^+(\X)$ is $\E_P[X] \ =\ \int_{\X} x\d P(x)$. %
The $n$-fold product measure of $P$ is denoted by $P^n=\otimes_{j=1}^n P$. Let $Q,P\in\mathcal M_{1}^+(\X)$, and let $Q$ absolutely continuous w.r.t.\ $P$ ($Q\ll P$, with the corresponding Radon-Nikodym derivative denoted by $\frac{\d Q}{\d P}$), %
their Kullback–Leibler divergence is defined as $\KL(Q\|P) =\int_{\X}\ln \!\left(\frac{\d Q(x)}{\d P(x)}\right)\d Q(x)$.
Given a probability measure $P \in \mathcal M_1^+(\R^d)$, we denote its moment of order $\bm \alpha \in \N^d$ as $M_{\bm \alpha}^P=\int_{\R^d} \b x^{\bm \alpha}\d P(\b x)$.
Normal random variables with mean $\bm \mu$ and covariance matrix $\bm \Sigma$ are denoted by $\mathcal N(\bm \mu, \bm \Sigma)$. The function $\psi_\P(\w) = \E_\P[e^{i \fip{X,\w}{2}}]$ is known as the characteristic function of $\P$.
A Hilbert space of functions $f: \X \to \R$ is a reproducing kernel Hilbert space (RKHS) $\H_k$ associated to a kernel $k\colon \X\times\X\to\R$ if $k(\cdot, x)\in\H_k$ for all $x\in\X$ and the reproducing property $f(x)=\langle f,k(\cdot, x)\rangle_{\H_k}$ holds for all $f\in\H_k$ and all $x \in \X$.\footnote{$k(\cdot,x)$ denotes the function $x'\in \X \mapsto k(x',x) \in \R$ while keeping $x \in \X$ fixed.} Let $\H_k^d=\H_k\times\cdots\times\H_k$ be the product RKHS with inner product $\langle\b f, \b g\rangle_{\H_k^d}=\sum_{j=1}^{d} \langle f_j, g_j\rangle_{\H_k}$ for $\b f=(f_j)_{j=1}^d, \b g=(g_j)_{j=1}^d \in \H_k^d$. For positive sequences $(a_n)_{n=1}^\infty$ and  $\left(b_n\right)_{n=1}^\infty$, (i) $a_n = \O(b_n)$ if there exist $ C>0$ and $n_0\in\Zp$ such that $a_n \le C b_n$ for all $n\ge n_0$,  (ii) $a_n = \Omega(b_n)$ if $b_n = \O(a_n)$, (iii) $a_n = \Theta(b_n)$ if $a_n = \O(b_n)$ and $b_n = \O(a_n)$. For a sequence of independent identically distributed (i.i.d.) real-valued random variables $(X_n)_{n=1}^\infty$, $X_n \sim P$ and a sequence of positive reals $(a_n)_{n=1}^\infty$ ($a_n>0$ for all $n$), $X_n=\O_P(a_n)$ means that $\left(\frac{X_n}{a_n}\right)_{n=1}^\infty$ is bounded in probability.

\section{KERNEL STEIN DISCREPANCY}\label{sec:KSD}
We now introduce our quantity of interest, the kernel Stein discrepancy (KSD). To simplify exposition, we split the presentation into the Langevin-Stein KSD (with domain $\X=\R^d$; Section~\ref{sec:langevin-stein}) and into the more abstract case of KSD on general domains $\X$ (Section~\ref{sec:general-ksd}); our results presented in Section~\ref{sec:results} are structured similarly.

\subsection{Langevin-Stein KSD on \texorpdfstring{$\R^d$}{R\^{}d}} \label{sec:langevin-stein}
Recall that we aim to compare a known and fixed distribution $P_0$ to an unknown distribution $P$, of which one obtains samples. %
Throughout this section, we assume that $P_0\in \mathcal{M}_1^+(\R^d)$ and $P\in \mathcal{M}_1^+(\R^d)$. Also, assume that $P_0$ and $P$ are absolutely continuous w.r.t.\ the Lebesgue measure with pdfs $p_0$ and $p$, respectively.
One can tackle this problem by constructing a goodness-of-fit measure, such as Langevin-Stein KSD \citep{chwialkowski16kernel, liu16kernelized}, which we detail below.

KSD is a specific IPM; %
indeed, considering $\mathcal F = \left\{ \mathcal A_ {p_0}\b f : \b f \in B\big(\H_k^d\big)\right\}$,
\begin{align}
\MoveEqLeft \KSD(P_0,P)= \sup_{f\in\mathcal F}\left|\E_{P_0}[f(X)] - \E_{P}[f(X)]\right| \\
&\hspace{-0.8cm}= \hspace{-0.15cm}\sup_{\b f \in B\left(\H_k^d\right)} \hspace{-0.18cm}  \left|\E_{P_0}[(\mathcal A_{p_0} \b f)(X)\,]
        -\E_{P}[(\mathcal A_{p_0} \b f)(X)]\right|,  \label{def:KSD}
\end{align}
where the operator $\mA_{p_0}$ is constructed to guarantee the mean-zero property (\citealt{gorham15measuring,chwialkowski16kernel, liu16kernelized})
\begin{align}
\E_{P_0}[(\mathcal A_{p_0}\b f)(X)\,]=0 \t{ for all } \b f \in B\!\left(\H_k^d\right); \label{eq:mean-zero-property}
\end{align}
this property, using the symmetry of $B\big(\H_k^d\big)$ [in other words, $\b f \in B\big(\H_k^d\big) \implies -\b f \in B\big(\H_k^d\big)$], simplifies \eqref{def:KSD} to
\begin{align}
\KSD(P_0,P)&= \sup_{\b f \in B\left(\H_k^d\right)} \E_{P}[(\mathcal A_{p_0} \b f)(X)]. \label{def:KSD-simplified}
\end{align}
One well-known operator satisfying \eqref{eq:mean-zero-property} is the so-called Langevin-Stein operator \citep{gorham15measuring, chwialkowski16kernel,liu16kernelized, oates17control,gorham17measuring},
defined for $\b f = (f_j)_{j=1}^d\in \H_k^d$ as
\begin{align}
(\mA_{p_0} \b f)(\b x)
= \langle \nabla_{\b x} \ln \!\big({p_0}(\b x)\big), \b f(\b x) \rangle_2
+ \sum_{j=1}^{d} \frac{\p f_j(\b x)}{\p x_j}. \;\; \label{eq:A-def}
\end{align}
Notice that the computation of $\mA_{p_0}$ relies on $\nabla_{\b x} \ln\! \big({p_0}(\b x)\big)$, hence one assumes that $p_0(\b x)>0$ for all $\b x \in \R^d$ (written shortly as $p_0>0$)---this dependence means that it is sufficient to know $p_0$ up to a constant multiplier---and that $p_0$ is differentiable. For \eqref{eq:mean-zero-property} to hold, one requires that $\lim_{\left\|\b x\right\|_2 \to \infty} h(\b x)p_0(\b x) = 0$ for all $h\in \H_k$ \citep[Lemma~2.2]{liu16kernelized}; for this condition it is sufficient if   $p_0$ is bounded and $\lim_{\left\|\b x\right\|_2 \to \infty}h(\b x) = 0$ for all $h\in \H_k$.

One can show  \citep[Theorem~2.2]{chwialkowski16kernel}  that $\KSD$ is a valid goodness-of-fit measure in the sense of
\begin{align}
\KSD(P_0,P) = 0 \iff P_0 = P  \label{eq:KSD-is-valid}
\end{align}
under mild conditions, particularly if the base kernel $k$ is $c_0$-universal \citep{carmeli10vector,sriperumbudur10relation}.
In the KSD construction (and throughout the paper when considering the Langevin-Stein KSD), we assume that the base kernel $k$ is twice continuously differentiable [$k\in \mathcal C^2(\R^d\times\R^d)$]. Indeed, regarding \eqref{eq:A-def}, by the reproducing property for kernel derivatives (\citealt[Theorem~1]{zhou08derivative}; \citealt[Lemma~1]{aubin22handling}), one can write $(\mA_{p_0} \b f)(\b x)$ as an inner product
\begin{align}
(\mA_{p_0} \b f)(\b x) &= \langle \b f, \b \xi_{p_0}(\b x) \rangle_{\H_k^d},    \label{eq:Stein-op-IP-form} \\
\H_k^d\ni\b \xi_{p_0}(\b x) &\coloneq \nabla_{\b x} \left[\ln\!\big({p_0}(\b x)\big)\right] k(\cdot,\b x)\hspace{-.05cm} + \hspace{-.05cm}\nabla_{\b x} k(\cdot,\b x), \;\;\; \label{eq:xi-p0-def}
\end{align}
for all $\b f \in \H_k^d$ and $\b x \in \R^d$,
which gives rise to the alternative form of KSD:
\begin{align}
\MoveEqLeft\KSD(P_0,P)
 \overset{\text{(a)}}{=} \sup_{\b f \in B\left(\H_k^d\right)} \E_{P}\!\left[\langle \b f, \b \xi_{p_0}(X) \rangle_{\H_k^d}\right]\\
&\hspace{-1.2cm}\overset{\text{(b)}}{=}\sup_{\b f \in B\left(\H_k^d\right)} \langle \b f, \E_{P} [\xi_{p_0}(X)] \rangle_{\H_k^d}
\overset{\text{(c)}}{=}\left\|\E_{P}[\xi_{p_0}(X)] \right\|_{\H_k^d}, \label{eq:KSD-xi}
\end{align}
where (a) is implied by \eqref{def:KSD-simplified} and \eqref{eq:Stein-op-IP-form}, (b) comes from swapping the inner product and the expectation \citep[(A.32)]{steinwart08support}, and (c) follows from the Cauchy–Bunyakovsky–Schwarz (CBS) inequality.

The Stein kernel $K_0: \R^d \times \R^d \to \R$ is defined based on $\xi_{p_0}$ as
$K_0(\b x,\b y) = \langle \b \xi_{p_0}(\b x), \b \xi_{p_0}(\b y) \rangle_{\H_k^d}$, ($\b x, \b y \in \R^d$) which, by \eqref{eq:xi-p0-def} and the reproducing property, takes the form
\begin{align}
K_0(\b x,\b y)
&= \big\langle \nabla_{\b x} \ln\!\big({p_0}(\b x)\big), \nabla_{\b y}\ln\!\big({p_0}(\b y)\big) \big\rangle_2 k(\b x,\b y) \\
&\quad + \big\langle \nabla_{\b y}\ln\!\big({p_0}(\b y)\big), \nabla_{\b x} k(\b x,\b y) \big\rangle_2 \\
&\quad + \big\langle \nabla_{\b x}\ln\!\big({p_0}(\b x)\big), \nabla_{\b y} k(\b x,\b y) \big\rangle_2 \\
&\quad + \sum_{j=1}^{d} \dfrac{\p^2 k(\b x,\b y)}{\p x_j \p y_j}.\label{eq:LS_KSD_K_0_expression}
\end{align}
We assume that $p_0 \in \mathcal C^1(\R^d)$, which, together with the assumed property that $k\in \mathcal C^2(\R^d\times\R^d)$, implies the continuity of $K_0$ and, in turn, the separability of $\H_{K_0}$ \citep[Lemma~4.33]{steinwart08support}. The following assumption summarizes our requirements for the Langevin-Stein KSD (i.e., the domain $\X = \R^d$).
\begin{assumption}[Langevin-Stein KSD]\label{ass:LS-KSD}
Let $P_0 \in \mathcal M_1^+\big(\R^d\big)$ and $k\in \mathcal C^2(\R^d\times\R^d)$. Assume that (i) $P_0$ is absolutely continuous w.r.t.\ the Lebesgue measure  with corresponding density $p_0$, (ii) $p_0$ is continuously differentiable: $p_0\in \mathcal C^1(\R^d)$, (iii) $p_0$ is positive: $p_0>0$, and (iv) $\lim_{\|\b x\|_{2} \to \infty} h(\b x) p_0(\b x) = 0$ for all $h \in \H_k$.
\end{assumption}

\subsection{General KSD}  \label{sec:general-ksd}

The construction in the preceding section can be extended to a topological space $(\X,\tau_\X)$ by considering $P_0,P \in \mathcal M_1^+(\X)$, $\H$ a Hilbert space of functions on $\X$, and $\Psi_{P_0}:\X \to \H$ such that the mean-zero property
\begin{align}
\E_{P_0}[\Psi_{P_0}(X)] = 0 \label{eq:gen-ksd-mean-zero} \noeqref{eq:gen-ksd-mean-zero}
\end{align}
holds.\footnote{The existence of the l.h.s.\ requires that $\E_{\P_0}\norm{\Psi_{\P_0}(X)}{\H} < \infty$ \citep[Theorem~2]{diestel77vector}.} One can then define the Stein operator $T_{P_0}$ on $\H$ as
\begin{align}
	\left(  T_{P_0} f \right)(x) = \fip{\Psi_{P_0}(x),f}{\H},\quad (f\in \H, x\in \X);\quad \label{eq:gen-ksd-operator-definition} \noeqref{eq:gen-ksd-operator-definition}
\end{align}
the operator inherits the mean-zero property \eqref{eq:gen-ksd-mean-zero}
\begin{align}
	\E_{P_0}\!\left[ \left( T_{P_0} f\right)(X) \right] &= \fip{\E_{P_0} [\Psi_{P_0}(X)],f}{\H} = 0, \label{eq:gen-ksd-operator-mean-zero} \noeqref{eq:gen-ksd-operator-mean-zero}
\end{align}
seen by interchanging the inner product with the expectation and  using that $\E_{P_0}[\Psi_{P_0}(X)] = 0$.
The KSD of $P_0$ (assumed to be fixed and known) and the sampling measure $P$ is then defined as the IPM
\begin{align}
	\MoveEqLeft\KSD(P_0,P) \coloneq \\
    &\sup_{f\in B(\H)}\big|\underbrace{\E_{P_0}[\left( T_{P_0} f \right)(X)]}_{=0} - \E_{P}[\left( T_{P_0} f \right)(X)]\big| \\
    &
	\overset{\text{(a)}}{=} \sup_{f\in B(\H)} \E_{P}[\left( T_{P_0} f \right)(X)]\label{eq:gen-ksd-tk0}\noeqref{eq:gen-ksd-tk0}\\
    &\stackrel{{\eqref{eq:gen-ksd-operator-definition}}}{=}\sup_{f\in B(\H)}  \E_{P}\fip{\Psi_{P_0}(X),f}{\H} \label{eq:gen-ksd-sup} \noeqref{eq:gen-ksd-sup}\\
    &\overset{\text{(b)}}{=} \norm{\E_{P} [\Psi_{P_0}(X)]}{\H} \label{eq:gen-ksd-norm-repr} \noeqref{eq:gen-ksd-norm-repr}\\
    & \stackrel{(c),(d),\eqref{eq:K_0-def}}{=} \sqrt{\E_{P\otimes P}[K_0(X,X')]}\\
    &\stackrel{(c),(d),(e)}{=}\norm{\int_\X K_0(\cdot,x)\d P(x)}{\H_{K_0}} \label{eq:population-ksd-equiv} %
\end{align}
(a) follows from the homogeneity of $T_{P_0}$ and the expectation, and using the symmetry of $B(\H)$. (b) follows as in \eqref{eq:KSD-xi}. We use that the norm in a Hilbert space is induced by its inner product in (c), the expectation and the inner product are swapped in (d) and the reproducing property \eqref{eq:repr-prop-of-K_0} implies (e); we also used the definition
\begin{align}
K_0(x,x') \coloneq \fip{\Psi_{P_0}(x),\Psi_{P_0}(x')}{\H} \quad (x,x' \in \X).\quad \label{eq:K_0-def}
\end{align}
As $K_0$ is a kernel, there exists an associated RKHS $\H_{K_0}$ for which $K_0$ is the (reproducing) kernel. Hence, for any $x,x'\in\X$ it holds that
\begin{align}
	K_0(x,x')%
	= \fip{\fm x, \fm{x'}}{\H_{K_0}}. \label{eq:repr-prop-of-K_0}
\end{align}
We note that $\Psi_{P_0}(x) \in \H$ and $K_0(\cdot,x) \in \H_{K_0}$ ($x\in\X$) but both yield the same Stein kernel $K_0$ [by \eqref{eq:K_0-def} and \eqref{eq:repr-prop-of-K_0}].

We collect our requirements for the general KSD in the following assumption.

\begin{assumption}[General KSD] \label{ass:integrability}
Assume that  $(\X,\tau_\X)$ is a topological space. Let $P_0 \in \mathcal M_1^+(\X)$ and $\Psi_{P_0} : \X \to \H$, where $\H$ is a Hilbert space. Let $K_0(x, y) = \fip{\Psi_{P_0}(x),\Psi_{P_0}(y)}{\H}$ for $x,y \in \X$. Suppose that (i) $\Psi_{P_0}$ is measurable, (ii) $\E_{P_0}[\Psi_{P_0}(X)] = 0$, and (iii) $\H_{K_0}$ is separable.
\end{assumption}

We note that the measurability of $x\mapsto \Psi_{P_0}(x)$ for all $x\in \X$ is sufficient to guarantee the measurability of $K_0$ and $K_0(\cdot,x)$ ($x\in \X$) by the assumed separability of $\H_{K_0}$ \citep[Lemma~4.25]{steinwart08support}.
Further, $\E_{P_0}[\Psi_{P_0}(X)] = 0$ implies that $\E_{P_0}[K_0(\cdot,X)] = 0$ by the equality of \eqref{eq:gen-ksd-norm-repr} and \eqref{eq:population-ksd-equiv}.

Taking $\X = \R^d$, $\H = \H_k^d$, and $\Psi_{P_0}(\b x) = \xi_{p_0}(\b x) = \nabla_{\b x}\big[\ln\!\big({p_0}(\b x)\big)\big]k(\cdot, \b x) + \nabla_{\b x}k(\cdot, \b x)\in \H^d_{k}$, where $\H_k$ is an RKHS with reproducing kernel $k:\R^d\times\R^d\rightarrow\R$, recovers the Langevin-Stein KSD on $\R^d$, derived independently in Section~\ref{sec:langevin-stein}.

Besides Langevin-Stein KSD, the general construction detailed in this section encompasses, for example, KSD on Riemannian manifolds and KSD on Hilbert spaces \citep[Example~2 and Example~3]{hagrass24stein}.
\begin{remark}
 KSDs can also be defined using the Pettis integral \citep{barp24targeted}, which, for simplicity, we do not consider in this paper.
\end{remark}

\section{KSD ESTIMATORS}\label{sec:KSD-estimators}

In this section, we recall two existing KSD estimators with established convergence rates alongside their computational complexity. Let $ X_{1:n} = (X_1,\dots,X_n)$ be an i.i.d.\ sample from $P$ (shortly, $X_{1:n}\sim P^n$) from which $\KSD(P_0,P)$ is estimated.

The squared KSD can be written in the form
\begin{align}
\MoveEqLeft \KSD^2(P_0,P) \stackrel{\eqref{eq:gen-ksd-norm-repr}}{=} \left\|\E_{P}[\Psi_{P_0}(X)]\right\|_{\H}^2\\
& \hspace{-0.7cm}\stackrel{(a), (b), (c), (b), (a)}{=} \left\|\E_{P}[ K_0(\cdot,X)] \right\|_{\H_{K_0}}^2\\&\stackrel{(a), (b), (d)}{=} \E_{P\otimes P}[K_0(X,X')].
\label{eq:ksd-sq-expect}
\end{align}
By making use of the fact that in a Hilbert space the norm is induced by the inner product in (a), swapping the
expectation and the inner product in (b),  using that
$K_0( x, y) = \langle  \Psi_{P_0}( x),  \Psi_{P_0}( y) \rangle_{\H_k^d} = \langle K_0(\cdot, x),  K_0(\cdot, y)\rangle_{\H_{K_0}}$
for all $x,  y \in \X$ in (c), and leveraging the reproducing property in (d). We refer to $x\in \X\mapsto K_0(\cdot, x) \in \H_{K_0}$
as the Stein feature map.

\tb{V-statistic estimator.}
Replacing $P$ in \eqref{eq:ksd-sq-expect} with the empirical measure $\hat{P}_n=\frac{1}{n}\sum_{j=1}^n \delta_{X_j}$ yields the V-statistic-based KSD estimator \citep{chwialkowski16kernel} $\widehat{\KSD}^2_V(P_0,P)\coloneq\KSD^2\!\big(P_0,\hat{P}_n\big) = \frac{1}{n^2}\sum_{i,j=1}^{n} K_0(X_i,X_j)$.
This estimator has runtime complexity $\O\!\left(n^2\right)$ and under a sub-Gaussian assumption on the Stein feature map, one can show \citep{kalinke25nystromksd} that it has a convergence rate
\begin{align}
\left|\widehat{\KSD}_V(P_0,P)-\KSD(P_0,P)\right|
= \O_{P^n}\!\left( n^{-1/2} \right).
\label{eq:vstat-rate}
\end{align}

\tb{Nyström-KSD estimator.}
Recently, the Nyström technique has been adapted to design an accelerated KSD estimator \citep{kalinke25nystromksd}. The idea of the approach is to consider a subsample (the sampling is carried out with replacement) $\{\{\tilde X_1,\ldots,\tilde X_m\}\}$ of the original sample $X_{1:n}$,
giving rise to the subspace
\begin{align}
\H_{K_0,m}
= \Sp\!\left( K_0\!\left(\cdot,\tilde{X}_j\right)\, :\, j \in [m] \right) \subset \H_{K_0}.
\label{eq:subspace}
\end{align}
This subspace is then used to approximate $\E_{\hat{\P}_n}[K_0(\cdot,X)]$ by taking the minimum norm solution of the optimization problem
\begin{align}
\min_{\bm\alpha=(\alpha_j)_{j=1}^m\in\R^m}
\left\|\E_{\hat{\P}_n}[K_0(\cdot,X)]
- \sum_{j=1}^{m} \alpha_j K_0\Big(\cdot,\tilde{ X}_j\Big)
\right\|_{\H_{K_0}}\hspace{-0.2cm},
\label{eq:proj-min}
\end{align}
attained by $\hat {\bm \alpha}$, resulting in the squared KSD estimator
\begin{align}
\widehat{\KSD}^2_N(P_0,P)
&= \left\| \sum_{j=1}^{m} \hat{\alpha}_j K_0\Big(\cdot,\tilde{ X}_j\Big) \right\|_{\H_{K_0}}^2.
\label{eq:proj-norm}
\end{align}
The estimator can be computed as
\begin{align}
\widehat{\KSD}^2_N(P_0,P)
&= \bm{\beta}^\T \b K_{m,m}^{-} \bm{\beta},
&\bm{\beta} &= \frac{1}{n} \b K_{m,n} \b 1_n \in \R^m, \notag
\end{align}
with the Gram matrices
\begin{align}
\b K_{m,m} &= \left[ K_0\!\left(\tilde{X}_a, \tilde{X}_b\right) \right]_{a,b=1}^m\in \R^{m\times m},\\
\b K_{m,n} &= \left[ K_0\!\left(\tilde{X}_a,  X_b\right) \right]_{a,b=1}^{m,n} \in \R^{m\times n}.
\end{align}
The runtime complexity of this estimator is $\O\!\left(mn + m^3\right)$. Under a sub-Gaussian assumption on the Stein feature map and given an appropriate spectral decay of its centered covariance operator and associated lower bound on $m$, the estimator achieves a convergence rate
\begin{align}
\left| \widehat{\KSD}_N(P_0,P) - \KSD(P_0,P) \right|
= \O_{P^n\otimes\Lambda^m}\!\left(n^{-1/2}\right),
\label{eq:nystrom-rate}
\end{align}
with $\Lambda^m$ encoding the Nyström sampling.

The main result of this paper is that no KSD estimator can achieve faster convergence rate than $n^{-1/2}$, specifically showing that the V-statistic and the Nyström-KSD estimators are rate-optimal.\footnote{One can also obtain the same $n^{-1/2}$ convergence rate (up to logarithmic factors) in the context of distribution compression with KSD \citep{li24debiased}.}

\section{MINIMAX ESTIMATION}\label{sec:minimax}
Before presenting our results, let us recall the framework of \tb{minimax estimation} in our context.
Our goal is to estimate $\KSD(P_0,P)$ based on samples $X_{1:n} \sim P^{n}$, given a target $P_0$. An estimator, denoted by $\hat{F}_n=\hat{F}_n(X_{1:n})$, is any (measurable) real-valued function of the observed data $X_{1:n}$ that approximates $\KSD(P_0,P)$.
The performance of an estimator $\hat{{F}}_n$ (referred to as risk) is defined as the expected absolute difference between the estimate and the true value:
\begin{align}
r_n\big(\hat{{F}}_n,P_0,P\big) &= \E_{P^{n}} \!\left| \hat{F}_{n}(X_{1:n}) - \KSD(P_0,P) \right|.
\end{align}
However, a good estimator should perform well not just for a single $P_0$ and $P$, but uniformly well over a range of plausible distributions. This leads to a worst-case analysis, where one considers the maximum risk of an estimator over a large class of $(P_0,P)$-pairs. Indeed, we let  $\mathcal T$ be the set of probability measures such that any $P_0\in \mathcal T$ satisfies Assumption~\ref{ass:LS-KSD} for a fixed base kernel $k$ in the case of Langevin-Stein KSD (resp.\ satisfies Assumption~\ref{ass:integrability} in the general case), guaranteeing that KSD is well-defined. To each $P_0$, we associate the sampling probability measures $\mathcal S_{P_0}$ for which $\KSD(P_0,P)$ is finite for any $P\in\mathcal S_{P_0}$:
\begin{align}
\mathcal S_{P_0} & \coloneq \{P\in \mathcal M_1^+(\X)\,: \,\KSD(P_0,P)<\infty\}\\
& \stackrel{(\dagger)}{=} \left\{ P \in {\mathcal M_1^+(\X)} : \E_{P} \sqrt{K_0(X,X)} < \infty \right\}.\;\  \label{eq:SP0-def}
\end{align}
$(\dagger)$ holds as by \eqref{eq:ksd-sq-expect} and the properties of the Bochner integral, one has that
\begin{align}
\KSD\left(P_0,P\right)  &= \left\|\E_{P} [K_0(\cdot,X)] \right\|_{\H_{K_0}} < \infty \iff \\
\infty &> \E_{P} \left\| K_0(\cdot,X)\right\|_{\H_{K_0}}\\
&\overset{\text{(a)}}{=}
\E_{P} \sqrt{\langle  K_0(\cdot,X), K_0(\cdot,X)\rangle_{\H_{K_0}}} \\
&\overset{\text{(b)}}{=} \E_{P} \sqrt{K_0(X,X)}, \label{eq:finiteness-of-KSD} \noeqref{eq:finiteness-of-KSD}
\end{align}
where (a) follows from the fact that in a Hilbert space the norm is induced by the inner product, and (b) is implied by the reproducing property.

The maximum risk of an estimator $\hat{F}_n$ is its worst-case performance over the $(P_0,P)$-pairs so constructed:
\begin{align}
R_n\big(\hat {{F}}_n\big) &= \sup_{P_0 \in \mathcal T}\sup_{P \in \mathcal{S}_{P_0}} r_n \big(\hat{{ F}}_n, P_0,P\big) \\
&\hspace{-1.2cm}= \sup_{P_0 \in \mathcal T }\sup_{P \in \mathcal{S}_{P_0}} \E_{P^{n}} \!\left| \hat{F}_{n}\left(X_{1:n}\right) - \KSD(P_0,P) \right|. \label{eq:F_performance}
\end{align}
Note that we require two supremums in \eqref{eq:F_performance} due to the valid $P$-s depending on the choice of $P_0$.

Finally, the minimax risk $R_n^*$ is the smallest possible maximum risk achievable by \textit{any} estimator. The term ``minimax'' reflects this two-step logic: one first takes the \emph{max}imum risk for a given estimator and then finds the estimator that \emph{min}imizes this maximum risk. Formally, it is the infimum of the maximum risk over all possible estimators $\hat{F}_n$:
\begin{align}
R_n^* & = \inf_{\hat{F}_n} R_n\big(\hat{F}_n\big)\\
&= \inf_{\hat{F}_n} \sup_{P_0 \in \mathcal T}\sup_{P \in \mathcal{S}_{P_0}} \E_{P^n}\! \left| \hat{F}_{n}\left(X_{1:n}\right) - \KSD(P_0,P) \right|.
\notag
\end{align}
The quantity $R_n^*$ represents the intrinsic statistical difficulty of the estimation problem and our goal is to establish a lower bound on $R_n^*$. To achieve this goal, we apply Markov's inequality, obtaining, for any $s_n>0$,
\begin{align}
    &s_n^{-1}R_n^*
    \ge\\
    &\; \inf_{\hat{F}_n} \sup_{P_0 \in \mathcal T}\sup_{P \in \mathcal{S}_{P_0}}P^n\!\Big(\underbrace{\left| \hat{F}_{n}\!\left(X_{1:n}\right) - \KSD(P_0,P) \right|}_{\eqcolon\hat\Delta_n} \ge s_n\Big),
    \label{eq:markov-reduction}
\end{align}
and control the r.h.s.\ using Le Cam's two-point method (outlined in Theorem \ref{theorem:le-cam-main-text}). In the next section, we establish a positive lower bound on \eqref{eq:markov-reduction} with $s_n=\Theta\!\left(n^{-1/2}\right)$, implying lower bounds for the minimax risk of KSD estimation.  Further, recalling from Section~\ref{sec:KSD-estimators} that known KSD estimation rates are $\O(s_n)$ with $s_n = n^{-1/2}$, our results settle the statistical optimality of these estimators.

\section{RESULTS}\label{sec:results}
Next, we present our lower bounds on the minimax estimation of KSD, both for the Langevin-Stein KSD on $\R^d$ (Section~\ref{sec:LS-KSD:minimax-lower-bound}) and for general domains (Section~\ref{sec:KSD:minimax-lower-bound}).

\subsection{Langevin-Stein KSD}\label{sec:LS-KSD:minimax-lower-bound}
In this section, we consider $\mathcal{X}=\R^d$ with the usual topology and $K_0$ as in \eqref{eq:LS_KSD_K_0_expression}. Before stating our result, we make the following assumption, which, with the continuity of $k$, implies that $k$ has a Bochner representation (detailed in Theorem~\ref{thm:bochner}).

\begin{assumption}[Langevin-Stein KSD; additional kernel assumptions]\label{ass:kernel_k_LS}
    Let $k:\R^d\times\R^d \rightarrow \R$ be a kernel.
    Assume that $k$ is translation-invariant ($\exists \text{ positive definite } \kappa$ such that $k(\b x, \b y) = \kappa(\b x - \b y)$ for all $\b x, \b y \in \R^d$).\footnote{Note that translation-invariance implies the boundedness of the kernel ($\sup_{\b x \in \R^d}\sqrt{k(\b x,\b x)} < \infty$).}
\end{assumption}
\begin{remark}
    Kernels satisfying Assumptions~\ref{ass:LS-KSD} and~\ref{ass:kernel_k_LS} are, for example, the IMQ kernel and the whole class of twice-differentiable Matérn kernels, in particular, Gaussian kernels.
\end{remark}

Our result on the minimax lower bound of Langevin-Stein KSD reads as follows.
\begin{theorem}[minimax lower bound of Langevin-Stein KSD]
  \label{thm:minimax}
  Suppose that Assumptions~\ref{ass:LS-KSD} and \ref{ass:kernel_k_LS} hold, and that $k$ is characteristic. %
  Let $\hat{F}_n$ be any estimator of $\D{\P}$ using $n \in \Zp$ samples from $\P \in \mathcal{S}_{\P_0}$ ($\P_0 \in \mathcal T$), where $ \mathcal{S}_{\P_0}$ is defined in \eqref{eq:SP0-def} with $\X = \R^d$.
  Then, there exists a universal constant $c > 0$ such that
    \begin{align}
        \inf_{\hat{F}_n} \sup_{P_0 \in \mathcal T}\sup_{P \in \mathcal{S}_{P_0}}P^n\Bigg(\hat\Delta_n \geq \frac{c}{\sqrt n}\Bigg)  > 0, \label{eq:main-result-langevin}
    \end{align}
    with $\hat \Delta_n$ as defined in \eqref{eq:markov-reduction}. In particular, by \eqref{eq:markov-reduction},
    $n^{1/2}c^{-1}R^*_n > 0$.
\end{theorem}

\begin{remark}\label{rem:remark1}~
\begin{enumerate}[label=(\roman*)]
    \item  Note that the characteristic property of $k$ is sufficient; we do not require $c_0$-universality [as discussed below \eqref{eq:KSD-is-valid}] for Theorem~\ref{thm:minimax} to hold.
    \item This result shows that the minimax lower bound of KSD estimation on $\R^d$ is $s_n = \Theta\big(n^{-1/2}\big)$, and specifically establishes the rate optimality of the V-statistic and Nyström-based KSD estimators given their matching rate of convergence recalled in Section~\ref{sec:KSD-estimators}.
\end{enumerate}
\end{remark}

For a Gaussian base kernel $k$, our following corollary makes the constant $c > 0$ explicit.
\begin{corollary}\label{cor:minimax-k-gaus}
  In the setting of Theorem~\ref{thm:minimax}, suppose that $k(\b x, \b y) = e^{-\gamma \norm{\b x - \b y}{2}^2}$ for some $\gamma > 0$ ($\b x, \b y \in \R^d$). Then, \eqref{eq:main-result-langevin}  holds with $c = (4\gamma + 1)^{-d/4}/2$.
\end{corollary}

Note that the constant $c$ presented in the corollary increases exponentially with the dimension $d$, highlighting that the difficulty of KSD estimation can increase exponentially with $d$.

\subsection{General KSD} \label{sec:KSD:minimax-lower-bound}
In this section, $(\X,\tau_\X )$ is a topological space and we impose the following additional assumption, ensuring that (i) KSD is valid for at least one $P_0$ and that (ii) $(\X,\tau_\X)$ is sufficiently equipped with continuous bounded functions (used throughout the proof; see Lemma~\ref{lem:phi-exists}).

\begin{assumption}[General KSD; weak validity]\label{ass:gen_ksd}
Assume that, for at least one $P_0\in \mathcal T$, KSD is valid in the sense of \eqref{eq:KSD-is-valid} for all $P\in \mathcal S_{P_0}$.  In other words, for all $P \in \mathcal S_{P_0}$, $P \ne  P_0$ iff.\ $\KSD(P_0,P) > 0$.
Further assume that there exists a $\varphi_0\in\mathcal C_b(\X)$ such that there exists no $c\in\R$ such that $\varphi_0=c$ holds $P_0$-almost surely.
\end{assumption}

Our minimax lower bound result for general KSD is as follows.
\begin{theorem}[minimax lower bound of general KSD]\label{thm:minimax-lower}
Let Assumptions~\ref{ass:integrability} and \ref{ass:gen_ksd}
hold. Then, there exists a constant $B>0$ such that
\begin{align}
    \liminf\limits_{n\to \infty}\inf_{\hat{F}_n} \sup_{P_0 \in \mathcal T}\sup_{P \in \mathcal{S}_{P_0}}  P^n\Bigg(\hat \Delta_n \geq\frac{B}{\sqrt{n}}\Bigg) > 0,
\end{align}
with $\hat{\Delta}_n$ as defined in \eqref{eq:markov-reduction}. In particular, by \eqref{eq:markov-reduction},
$\liminf\limits_{n\to \infty} n^{1/2} B^{-1} R_n^* > 0$.
\end{theorem}
\begin{remark}\label{rem:remark2}~
    \begin{enumerate}[label=(\roman*)]
        \item This result shows that the minimax lower bound of KSD estimation on a general topological space $(\X,\tau_\X )$ is $n^{-1/2}$, given Assumptions~\ref{ass:integrability} and  \ref{ass:gen_ksd}; in other words, no KSD estimator can achieve a faster rate in the minimax sense.
        \item We also note that the bound in Theorem~\ref{thm:minimax} is achieved for any $n \in \Zp$, while Theorem~\ref{thm:minimax-lower} provides an asymptotic bound for the risk.
    \end{enumerate}
\end{remark}

We proceed by sketching the main ideas of the proofs of our main results (Theorem~\ref{thm:minimax} and Theorem~\ref{thm:minimax-lower}), with the full proofs deferred to the appendices.

\section{PROOF SKETCHES} \label{sec:proof-sketches}

Both of our results use Le Cam's two-point method. The core idea of this technique is to reduce the problem of finding a lower bound over a large class of distributions $P_0 \in \mathcal T$ and $P \in \mathcal S_{P_0}$ to the problem of finding a carefully crafted adversarial sequence of distributions; the key technical challenge and one contribution of our work is the construction of this adversarial sequence. Le Cam's two-point approach, following directly from \citet[(2.9) and Theorem~2.2]{tsybakov09introduction}, is as follows.

\begin{theorem}[Theorem 2.2; \citealt{tsybakov09introduction}]
\label{theorem:le-cam-main-text}
Let $\mathcal Y$ be a measurable space, $(\Theta,d)$ a semi-metric space, and $\mathcal P_{\Theta} =
\{\P_\theta : \theta \in  \Theta\}$ a class of probability measures
on $\mathcal Y$ indexed by $\Theta$. We
observe data $D \sim \P_{\theta} \in \mathcal P_{ \Theta}$ with some unknown
parameter $\theta$. The goal is to estimate $\theta$. Let $\hat \theta = \hat \theta(D)$ be an estimator of $\theta$ based on~$D$.
Assume that there exist $\theta_0,\theta_1\in \Theta$ such that
$d(\theta_0,\theta_1) \ge 2s > 0$ and $\mathrm{KL}(\P_{\theta_1}||\P_{\theta_0})
\le \alpha <\infty$ for $\alpha > 0$. Then
\begin{align*}
  \inf_{\hat \theta}\sup_{\theta\in \Theta}\P_\theta\big(d\big(\hat \theta,\theta\big) \ge s\big) \ge f(\alpha),
\end{align*}
with $f(\alpha) \coloneq \max \!\big\{\exp(-\alpha) / 4,(1-\sqrt{\alpha/2})\big\}>0$.
\end{theorem}

We now elaborate the main ideas behind our results.

\subsection{Proof Sketch for Theorem~\ref{thm:minimax}}
After recalling from \eqref{eq:markov-reduction} that \begin{align}
    R^*_n \geq\inf_{\hat{F}_n}\sup_{\P_0 \in \mathcal T}\sup_{\P \in \mathcal{S}_{\P_0}} \P^n\big(\hat\Delta_n\geq C\big)
\end{align}
by Markov's inequality, and noticing that all Gaussian distributions are in $\mathcal T$ for a bounded $k$, we first obtain the bound
$R^*_n\geq \inf_{\hat F_n} \sup_{\P \in \mathcal S_{P_0}} \P^n\big(\hat\Delta_n\geq C\big)$,
with $P_0 = \mathcal{N}\big(\bm 0_{d}, \b I_d\big)$ now fixed. Using this probabilistic form, we proceed by applying Le Cam's method. In our case, this boils down to designing an adversarial distribution pair $(\mathbb P,\mathbb Q)$ such that
\begin{equation}
    \big|\KSD(P_0,\mathbb P)-\KSD(P_0,\mathbb Q)\big|\geq \frac{2c}{\sqrt n} \label{eq:ksd-distance-proof-sketch}
\end{equation}
while $\KL\big(\mathbb Q^n\|\mathbb P^n\big)\leq\alpha$ with $0 < c,\alpha < \infty$.

To achieve this goal, we let $\mathbb P = \mathcal{N}\big(n^{-1/2}\b e_j,\b I_d\big)$ and $\mathbb Q =\mathcal{N}\big(\bm 0_{d},\b I_d\big) = P_0$. Notice that, in this case, $\KSD(P_0,\mathbb Q) = \KSD(P_0,P_0) =0$, and thus \eqref{eq:ksd-distance-proof-sketch} reduces to $\KSD(P_0,\mathbb P) \ge 2c/\sqrt{n}$.

\tb{Controlling the distance.}  To control the distance, we rely on two auxiliary lemmas.
Our first lemma shows that in case of a standard normal target $P_0$, $\D\P$ can be expressed in terms of the characteristic function of the sampling distribution $\P$, if $\P$ satisfies weak moment conditions.

\begin{lemma}[KSD in terms of characteristic functions]\label{lemma:lsksd-closed-form}
  Suppose that Assumption~\ref{ass:kernel_k_LS} holds. Further assume that $k\in \mathcal{C}^2(\R^d\times\R^d)$. Let $k$ have Bochner representation $k(\b x,\b y) = \int_{\R^d}e^{-\imag\fip{\b x - \b y, \bm \omega}{2}}\d \Lambda(\bm \omega)$. Let $\P_0 = \mathcal N(\bm 0_d,\b I_d) \in \mathcal M_1^+{(\R^d)}$ and suppose $\P \in \mathcal{M}_1^+{(\R^d)}$ is such that $M_{\bm \alpha}^{{\P}} < \infty$  for all $|\bm \alpha| \leq 2$ ({$\bm \alpha \in \N^d$}). Then it holds that
  \begin{align}
      \Ds\P & = \int_{\R^d} \norm{\nabla_{\w} \psi_P(\w) + \w\psi_P(\w)}{\C^d}^2\d \Lambda(\w). \notag
  \end{align}
\end{lemma}
\begin{remark}
    \citet[Example 4.2]{wynne24spectral} contains a similar expression of the KSD. However, their formula relies on the characteristic functions of the target ($\psi_{P_0}$) and the sampling distribution ($\psi_P$).
\end{remark}
If $\P$ is a multivariate Gaussian, Lemma~\ref{lemma:lsksd-closed-form} simplifies as shown in the following corollary.
\begin{lemma}[Lemma~\ref{lemma:lsksd-closed-form} with $P=\mathcal{N}(\bm \mu, \bm \Sigma)$]\label{lemma:lsksd-closed-form-p-gaussian}
    In the setting of Lemma~\ref{lemma:lsksd-closed-form}, let $\P = \mathcal{N}(\bm \mu, \bm \Sigma)$. Then KSD  is
      \begin{align}
          \MoveEqLeft\Ds\P=\\& \int_{\mathbb{R}^d}\left( \norm{\bm{\mu}}{2}^2 + \norm{\bm \omega-\bm \Sigma\bm \omega}{2}^2 \right)\norm{\psi_\P(\bm \omega)}{\C}^2\d\Lambda(\bm \omega). \notag
      \end{align}
\end{lemma}
Therefore, by using that $\KSD(P_0,\mathbb P) = 0$ and invoking Lemma \ref{lemma:lsksd-closed-form-p-gaussian}, we obtain
\begin{align}
    \KSD(P_0,\mathbb P) = n^{-1}\int_{\R^d}\norm{\psi_{\mathbb P}(\w)}{\C}^2\d\Lambda(\w),
\end{align}
which, after establishing the positivity of the integral (by the characteristic property of $k$) and taking the positive square root on both sides, implies \eqref{eq:ksd-distance-proof-sketch}.

\tb{Controlling the KL divergence.} Utilizing the known expressions for the KL divergence of product measures and the KL divergence of Gaussians (recalled in Lemma~\ref{lemma:tsy-kl-prod-measure} and Lemma~\ref{lemma:kl-gaussians}, respectively) yields $\KL(\mathbb Q^n\|\mathbb P^n)\leq 1/2 \eqcolon \alpha$ for all $n\in \Zp$.

We conclude by invoking Theorem \ref{theorem:le-cam-main-text} using both controlled quantities.

\subsection{Proof Sketch for Theorem~\ref{thm:minimax-lower}}
Let $P_0$ be as in Assumption~\ref{ass:gen_ksd}. The proof starts by observing that $\{P_0\}\subset \mathcal T$ implies\begin{align}
    R^*_n&=\inf_{\hat{F}_n}\sup_{P_0\in \mathcal T} \sup_{P\in \mathcal{S}_{P_0}}C^{-1}\E_{P^n}\big[\hat \Delta_n\big] \\
         &\geq \inf_{\hat{F}_n} \sup_{P\in \mathcal{S}_{P_0}} \P^n\big(\hat \Delta_n\geq C\big).
\end{align}
Then, we obtain a lower bound by applying Le Cam's method with the adversarial distribution pair $\mathbb P = P_0$ and $\mathbb Q = P_n$. $P_n$ is defined as a perturbation of $P_0$:\footnote{Similar perturbations are known in the testing literature \citep{anderson94twosample}.}
\begin{align}
    P_n(A) = \int_{A} 1 + \epsilon_n \varphi(x)\d P_0(x),\quad \forall A \in \mathcal B(\X), \label{eq:P_n-definition}
\end{align}
where $\varphi\in \mathcal C_b(\X)$, $\E_{P_0}[\varphi(X)] = 0$, $\varphi$ is not identically zero $P_0$-almost surely, and $\epsilon_n = cn^{-1/2}$ with $c>0$.\footnote{The existence of such $\varphi$ is guaranteed by Lemma~\ref{lem:phi-exists}.}

We start by showing that $P_n$ belongs to $\mathcal{S}_{P_0}$.
Then, to apply Le Cam's method, we establish that $|\KSD(P_0,\mathbb P) - \KSD(P_0,\mathbb Q)|\geq 2cn^{-1/2}$ and that $\KL(\mathbb P^n\|\mathbb Q^n)\le \alpha$, with $0 < \alpha < \infty$.

\tb{$P_n$ is a probability measure.}  $P_n(\X) = 1$ holds by the definition of $P_n$.
To show that $P_n$ is non-negative, it suffices to note that
\begin{align}
1+\epsilon_n\varphi(x) \geq 1+\epsilon_n L \text{ for all } x \in \X, \label{eq:sketch-pn-def}
\end{align}
where $L = \inf_{x\in \X}\varphi(x)\in(-\infty,0]$. The r.h.s.\ of \eqref{eq:sketch-pn-def} is non-negative for $n$ large enough; hence, $\P_n$ is non-negative for $n$ large enough.

\tb{KSD$(P_0,P_n) < \infty$.}
Rewriting $\E_{P_n}\sqrt{K_0(X,X)} = \int_{\X}\sqrt{K_0(x,x)} \d P_0(x) +\epsilon_n\int_{\X}\sqrt{K_0(x,x)}\varphi(x) \d P_0(x)$, the first integral is finite by \eqref{eq:finiteness-of-KSD}; the finiteness of the second term follows by using that $\varphi$ is bounded, \eqref{eq:finiteness-of-KSD}, and $\epsilon_n < \infty$.
As $P_n \in \mathcal M_1^+(\X)$ and $\KSD(P_0,P_n)< \infty$, we have shown that $P_n\in \mathcal S_{P_0}$.

\tb{Controlling the distance.} Recall from \eqref{eq:gen-ksd-norm-repr} that for all $P\in \mathcal S_{P_0}$, $\KSD(P_0,P) = \|\E_{P}\Psi_{P_0}(X)\|_{\H}$.  Our specific choice of $P_n$ [\eqref{eq:P_n-definition}] allows to write
$\KSD(P_0,P_n) = \epsilon_n\|\E_{P_0}\varphi(X)\Psi_{P_0}(X)\|_{\H} \eqcolon \epsilon_n C_{\varphi} > 0$, where the positivity follows from (i) the assumed validity of KSD in Assumption~\ref{ass:gen_ksd} and (ii)  $\epsilon_n > 0$.

\tb{Controlling the KL divergence.} The definition of $P_n$ implies that
\begin{align}
\KL(P_n\|P_0) = \E_{P_0}\big[\big(1+\epsilon_n\varphi(X)\big)\ln\big(1+\epsilon_n\varphi(X)\big)\big].
\end{align}
Then, by the fact that $\ln(1+x)\leq x,$ when $x>-1$, we obtain the bound on the KL divergence
\begin{align}
\KL(P_n\|P_0) \hspace{-0.1cm}\leq \epsilon_n\underbrace{\E_{P_0}[\varphi (X)]}_{=0} + \varepsilon_n^2\underbrace{\E_{P_0}[\varphi^2(X)]}_{\eqcolon M}=cn^{-1}M.\notag
\end{align}
Therefore, by the formula of the KL divergence of product measures (Lemma~\ref{lemma:tsy-kl-prod-measure}), we get the bound $\KL(\mathbb Q^n, \mathbb P^n)=\KL(P_n^n\|P_0^n) = n \KL(P_n\|P_0)  \leq cM<\infty$.

The proof concludes by invoking Theorem \ref{theorem:le-cam-main-text} with both controlled quantities.

\subsubsection*{Acknowledgments}
The authors thank Lester Mackey for his constructive comments.
FK is supported by the pilot program Core-Informatics of the Helmholtz Association (HGF). %

\setcounter{equation}{0}
\renewcommand{\theequation}{\thesection.\arabic{equation}} %

\appendix
\onecolumn

\section{PROOFS}
This section is dedicated to the proofs of our statements in the main text. The proof of Lemma~\ref{lemma:lsksd-closed-form} is in Appendix~\ref{sec:proof-lsksd-closed-form}, that of Lemma~\ref{lemma:lsksd-closed-form-p-gaussian} is in Appendix~\ref{sec:proof-lsksd-closed-form-p-gaussian}, and that of Theorem~\ref{thm:minimax} is in Appendix~\ref{sec:proof-minimax}. Corollary~\ref{cor:minimax-k-gaus} is proved in Appendix~\ref{sec:proof-cor-minimax-gauss}. We prove the general KSD lower bound (Theorem~\ref{thm:minimax-lower}) in Appendix~\ref{sec:proof-minimax-lower}.

\subsection{Proof of Lemma~\ref{lemma:lsksd-closed-form}}\label{sec:proof-lsksd-closed-form}
Recall that $k(\b x,\b y) = \int_{\R^d}e^{-\imag\fip{\b x - \b y, \bm \omega}{2}}\d \Lambda(\bm \omega)$ by Theorem~\ref{thm:bochner}. Let $\Lambda' = \Lambda/\Lambda(\R^d)$ and note that $\Lambda' \in \mathcal M_1^+(\R^d)$. We first show that \begin{equation}
M_{\bm \alpha}^{\Lambda'}<\infty,\label{eq:finite_moments}
\end{equation} with $|\bm \alpha| \le 2$ and $\bm \alpha \in \N^d$, which we will use multiple times throughout the remaining proof. Indeed, notice that $k(\b x,\b y) = \kappa(\b x - \b y) = \int_{\R^d}e^{-\imag\fip{\b x-\b y, \w}{2}}\d\Lambda(\w) = \Lambda(\R^d)\psi_{\Lambda'}(\b y-\b x)$; hence, 
\begin{align}
k \in \mathcal{C}^2(\R^d\times\R^d)\implies \kappa\in\mathcal C^2(\R^d)\implies \psi_{\Lambda'}\in \mathcal C^2(\R^d),\label{eq:charact_is_c2}
\end{align}
where the first implication holds by the composition of functions. Given that $\psi_{\Lambda'}\in \mathcal C^2(\R^d)$, the application of Theorem~\ref{thm:exist-of-moments} now yields \eqref{eq:finite_moments}.

To obtain the expression presented in Lemma~\ref{lemma:lsksd-closed-form}, we rewrite KSD as
\begin{align}
\MoveEqLeft\Ds\P \overset{\text{(a)}}{=} \E_{\P\otimes \P}K_0(X,Y) \\
&\overset{\text{(b)}}{=} \int_{\R^d\times \R^d} \fip{\nabla_{\b x} \log p_0(\b x),\nabla_{\b y}\log p_0(\b y)}{\C^d} k(\b x, \b y) + \fip{\nabla_{\b x}\log p_0(\b x),\nabla_{\b y}k(\b x, \b y)}{\C^d} \\
&\qquad+\fip{\nabla_{\b y} \log p_0(\b y),\nabla_{\b x}k(\b x, \b y)}{\C^d} + \sum_{j=1}^{d}\frac{\partial^2k(\b x, \b y)}{\partial x_j \partial y_j} \d (\P\otimes\P)(\b x, \b y) \\
&\overset{\text{(c)}}{=}\underbrace{ \int_{\R^d\times \R^d} \fip{\b x,\b y}{\C^d} k(\b x, \b y) \d (\P\otimes\P)(\b x, \b y)}_{\eqcolon t_1}  - \underbrace{\int_{\R^d\times \R^d}  \fip{\b x,\nabla_{\b y}k(\b x, \b y)}{\C^d} \d (\P\otimes\P)(\b x, \b y)}_{\eqcolon t_2}  \\
&\qquad-\underbrace{\int_{\R^d\times \R^d}\fip{\b y,\nabla_{\b x}k(\b x, \b y)}{\C^d} \d (\P\otimes\P)(\b x, \b y)}_{\eqcolon t_3} + \underbrace{\int_{\R^d\times \R^d}\sum_{j=1}^{d}\frac{\partial^2k(\b x, \b y)}{\partial x_j \partial y_j} \d (\P\otimes\P)(\b x, \b y)}_{\eqcolon t_4} \\
& \overset{\text{(d)}}{=}  \underbrace{\int_{\R^d} \fip{\nabla_{\w}\psi_\P(-\w),\nabla_{\w}\psi_\P({-}\w)}{\C^d} \d \Lambda(\w)}_{\overset{\eqref{eq:deriv-t1}}{=}t_1} + \underbrace{\int_{\R^d} \fip{\nabla_{\w}\psi_{\P}(\w),\w}{\C^d} \psi_{\P}(-\w)\d\Lambda(\w)}_{\overset{\eqref{eq:deriv-t2}}{=}{-}t_2}  \\
& \quad  + \underbrace{\int_{\R^d} \fip{\nabla_{\w}\psi_{\P}( -\w),\w}{\C^d} \psi_{\P}(\w)\d\Lambda(\w)}_{\overset{\eqref{eq:deriv-t3}}{=}{-}t_3} +  \underbrace{\int_{\R^d} \norm{\w}{2}^2\psi_\P(-\w) \psi_\P(\w) \d \Lambda(\w)}_{\overset{\eqref{eq:deriv-t4}}{=}t_4}\\
&\overset{\text{(e)}}{=} \int_{\R^d} \fip{ \nabla_{\w} \psi_P(-\w), \nabla_{\w}\psi_P(-\w)}{\C^d} + \fip{ \nabla_{\w}\psi_P(\w), \w }{\C^d}\psi_{\P}( -\w)\\
&\qquad+\fip{\nabla_{\w} \psi_P(-\w), \w}{\C^d}\psi_{\P}( \w) +{\norm{\w}{2}^2}   \psi_\P(\w)\psi_\P(-\w)\d \Lambda(\w)\label{eq:lemma1-e}
\overset{\text{(f)}}{=} \int_{\R^d}\norm{\nabla_{\w} \psi_P(\w) + \w\psi_P(\w)}{\C^d}^2\d \Lambda(\w),
\end{align}
with the following details. In (a), we use the definition of $\Ds\P$ and in (b) the definition of $K_0$ [\eqref{eq:LS_KSD_K_0_expression}]. Note that $\P_0$ has (Lebesgue) density $p_0(\b x) \propto e^{-\norm{\b x}{2}^2/2}$ by assumption; to obtain (c), we use that $\nabla_{\b x} \log p_0(\b x) = -\b x$ (resp.\ $\nabla_{\b y}\log p_0(\b y) = -\b y$) together with linearity of the inner product and the expectation. We tackle terms $t_1$--$t_4$ separately below [in particular, we verify that we can (i) apply Fubini's theorem and (ii) flip the order of integration and differentiation, respectively] and combine them afterwards, to obtain (d).  (e) follows from the linearity of the integration. Properties of the norm on $\C^d$ yield (f), as we show in the following. Indeed, abbreviate $\b z = \nabla_{\w}\psi_P(\w)\in \C^d$, $ \w \in \R^d$, and $c ={\psi_P(\w)} \in \C$. Then it follows that 
\begin{align}
    \MoveEqLeft \norm{\nabla_{\w} \psi_P(\w) + \w\psi_P(\w)}{\C^d}^2  = \norm{\b z + c \w}{\C^d}^2\\
    &\overset{\text{(a)}}{=} \norm{\b z}{\C^d}^2 + \norm{c\w}{\C^d}^2 + \fip{\b z, c\w}{\C^d} + \fip{ c\w, \b z}{\C^d} \\
    &\overset{\text{(b)}}{=} \norm{\b z}{\C^d}^2 + \norm{c\w}{\C^d}^2 + \fip{\b z, \w}{\C^d} \overline c +  \fip{ \w, \b z}{\C^d}c \\
    &\overset{\text{(c)}}{=} \norm{\b z}{\C^d}^2 + \norm{\w}{\C^d}^2\|c\|^2_{\C} + \fip{\b z, \w}{\C^d} \overline c +  \fip{ \w, \b z}{\C^d}c \\
    &\overset{\text{(d)}}{=} \norm{\b z}{\C^d}^2 + \norm{\w}{\C^d}^2\|c\|^2_{\C} + \fip{\b z, \w}{\C^d} \overline c +  \fip{\b{\co z}, \w}{\C^d}c
    \\
    &\overset{\text{(e)}}{=} \norm{\nabla_{\w}\psi_P(\w)}{\C^d}^2 + \|\w\|_{2}^2\norm{\psi_P(\w)}{\C}^2 + \fip{\nabla_{\w}\psi_P(\w),\w}{\C^d}\psi_{P}(-\w) + \fip{\nabla_{\w}\psi_P(-\w),\w}{\C^d}\psi_P(\w).
\end{align}
In (a), we used that the norm in $\C^d$ is induced by the inner product, (b) follows from linearity in the 1st argument and the conjugate-linearity in the 2nd argument of the complex inner product, (c) is implied by the homogeneity of norms, (d) holds by $\fip{\w, \b z}{\C^d} = \b z^* \w = \sum_{j\in [d]}\co{z}_j\omega_j = \sum_{j\in [d]} \co{\omega}_j \co{z}_j = \w^* \co{\b z} = \fip{\co{\b z},\w}{\C^d}$ using that $\w \in \R^d$, and, in (e), we substituted the abbreviated quantities and used that $\psi_P(-\w) = \co{\psi_P(\w)}$.

\tb{Term $t_1$.} 
We rewrite the first term as 
\begin{align}
    t_1 &\overset{\text{(a)}}{=} \int_{\R^d\times \R^d} \fip{\b x,\b y}{\C^d} \underbrace{\int_{\R^d}e^{-\imag\fip{\b x -\b y, \w}{2}}\d \Lambda(\w)}_{=k(\b x,\b y)} \d (\P\otimes\P)(\b x, \b y) \\
    &\overset{\text{(b)}}{=} \int_{\R^d} \int_{\R^d\times \R^d} \fip{\b x,\b y}{\C^d} e^{-\imag\fip{\b x -\b y, \w}{2}}\d (\P\otimes\P)(\b x, \b y) \d \Lambda(\w) \\
    &\overset{\text{(c)}}{=} \int_{\R^d} \int_{\R^d\times \R^d} \fip{\b x e^{-\imag\fip{\b x, \w}{2}},\b y e^{{-}\imag\fip{\b y, \w}{2}}}{\C^d} \d (\P\otimes\P)(\b x, \b y) \d \Lambda(\w) \\
    &\overset{\text{(d)}}{=}  \int_{\R^d} \fip{\int_{\R^d}\b x e^{-\imag\fip{\b x, \w}{2}}\d \P(\b x),\int_{\R^d}\b y e^{{-}\imag\fip{\b y, \w}{2}}\d \P(\b y)}{\C^d} \d \Lambda(\w) \\
    &\overset{\text{(e)}}{=} \int_{\R^d} \fip{i\nabla_{\w}\psi_\P(-\w),i\nabla_{\w}\psi_\P({-}\w)}{\C^d} \d \Lambda(\w) \\
    &\overset{\text{(f)}}{=} \int_{\R^d} \fip{\nabla_{\w}\psi_\P(-\w),\nabla_{\w}\psi_\P({-}\w)}{\C^d} \d \Lambda(\w), \label{eq:deriv-t1}
\end{align}
where Bochner's theorem (recalled in Theorem~\ref{thm:bochner}) implies (a). In (b), we use the linearity of the integral and apply Fubini's theorem to change the order of integration, which we validate in Lemma~\ref{lemma_aux:fubini}(i) after recalling \eqref{eq:finite_moments}. The properties of the exponential function ($e^{-i\ip{\b x - \b y}{\w}_2} = e^{-i\ip{\b x}{\w}_2} e^{i\ip{\b y}{\w}_2}$), the conjugate-linearity of the complex inner product in the second argument with the fact that $\co{e^{iz}}=e^{-iz}$ for $z\in \R$,  and the linearity of the inner product in the first argument yield (c). The integrals were swapped with the inner product in (d). Invoking Lemma~\ref{lemma_aux:derivative_of_char} [validated in \eqref{eq:charact_is_c2}] on both arguments of the inner product yields (e), the linearity and the conjugate-linearity of the complex inner product in the first and the second argument, respectively, and using that $i\co{i}=-i^2=1$ give (f).

\tb{Term $t_2$.} We obtain the alternative expression of the second term
\begin{align}
    t_2 &\overset{\text{(a)}}{=} \int_{\R^d\times \R^d} \bigfip{\b x, \nabla_{\b y} \underbrace{\int_{\R^d}e^{-\imag\fip{\b x -\b y, \w}{2}}\d \Lambda(\w)}_{=k(\b x,\b y)}}{\C^d}  \d (\P\otimes\P)(\b x, \b y) \\
    &\overset{\text{(b)}}{=} \int_{\R^d\times \R^d} \fip{\b x, i\int_{\R^d}\w e^{-\imag\fip{\b x - \b y, \w}{2}}\d \Lambda(\w)}{\C^d}  \d (\P\otimes\P)(\b x, \b y)  \\
    &\overset{\text{(c)}}{=} \int_{\R^d\times \R^d} \int_{\R^d} \fip{\b x, i\w e^{-\imag\fip{\b x - \b y, \w}{2}}}{\C^d} \d \Lambda(\w) \d (\P\otimes\P)(\b x, \b y)  \\
    &\overset{\text{(d)}}{=} \int_{\R^d \times \R^d}\int_{\R^d} -i\fip{\b x, \w}{\C^d}e^{\imag\fip{\b x, \w}{2}}e^{ -\imag\fip{\b y, \w}{2}} \d\Lambda(\w) \d (\P\otimes\P)(\b x,\b y)  \\ 
    &\overset{\text{(e)}}{=} \int_{\R^d}\int_{\R^d}\int_{\R^d} - i\fip{\b x e^{\imag\fip{\b x, \w}{2}}, \w}{\C^d}e^{-\imag\fip{\b y, \w}{2}} \d\P(\b y) \d\P(\b x) \d\Lambda(\w) \\
    &\overset{\text{(f)}}{=} \int_{\R^d}\int_{\R^d} - i\fip{\b x e^{\imag\fip{\b x, \w}{2}}, \w}{\C^d} \int_{\R^d} e^{-\imag\fip{\b y, \w}{2}} \d\P(\b y) \d\P(\b x) \d\Lambda(\w) \\
     &\overset{\text{(g)}}{=} \int_{\R^d} -i\bigfip{\underbrace{\int_{\R^d}\b x e^{\imag\fip{\b x, \w}{2}}\d\P(\b x)}_{\overset{\text{Lemma~\ref{lemma_aux:derivative_of_char}}\text{(i)}}{=}-\imag\nabla_{\w}\psi_\P(\w)}, \w}{\C^d}\psi_\P(-\w)   \d\Lambda(\w)
    \overset{\text{(h)}}{=} \int_{\R^d} -\fip{\nabla_{\w}\psi_{\P}(\w),\w}{\C^d} \psi_{\P}(-\w)\d\Lambda(\w), \hspace{.7cm} \label{eq:deriv-t2}
\end{align}
where (a) follows by Bochner's theorem (recalled in Theorem~\ref{thm:bochner}) and (b) is shown in Lemma~\ref{lemma_aux:derivative_of_kernel}(ii). In (c), we swap the inner product with the integral. The conjugate-linearity of the complex inner product in its second argument, the facts that $\co{e^{iz}}=e^{-iz}$ ($z\in \R$) and $e^{\imag\fip{\b x - \b y, \w}{2}} = e^{\imag\fip{\b x, \w}{2}} e^{-\imag\fip{\b y, \w}{2}}$ are used in (d). For (e), it suffices to apply Fubini's theorem, validated in Lemma~\ref{lemma_aux:fubini}(ii) by using \eqref{eq:finite_moments}, to use the product structure of $\P\otimes \P$, and then the linearity of the complex inner product in its first argument. The linearity of the integration gives (f). By the definition of characteristic function, the linearity of integration, and exchanging the inner product and the integral, we obtain (g).  Lemma~\ref{lemma_aux:derivative_of_char}(i) [validated in \eqref{eq:charact_is_c2}], the linearity of the complex inner product in the first argument, and $\imag^2 = -1$ yield (h).

\tb{Term $t_3$.} Similarly to $t_2$, we have
\begin{align}
    t_3 &\overset{\text{(a)}}{=} \int_{\R^d\times \R^d} \bigfip{\b y, \nabla_{\b x} \underbrace{\int_{\R^d}e^{-\imag\fip{\b x -\b y, \w}{2}}\d \Lambda(\w)}_{=k(\b x,\b y)}}{\C^d}  \d (\P\otimes\P)(\b x, \b y) \\
    &\overset{\text{(b)}}{=} \int_{\R^d\times \R^d} \fip{\b y, -i\int_{\R^d}\w e^{ -\imag\fip{\b x - \b y, \w}{2}}\d \Lambda(\w)}{\C^d}  \d (\P\otimes\P)(\b x, \b y)  \\
    &\overset{\text{(c)}}{=} \int_{\R^d\times \R^d} \int_{\R^d} \fip{\b y, -i\w e^{ -\imag\fip{\b x - \b y, \w}{2}}}{\C^d} \d \Lambda(\w) \d (\P\otimes\P)(\b x, \b y)  \\
    &\overset{\text{(d)}}{=} \int_{\R^d \times \R^d}\int_{\R^d} i\fip{\b y, \w}{\C^d}e^{\imag\fip{\b x, \w}{2}}e^{ -\imag\fip{\b y, \w}{2}} \d\Lambda(\w) \d (\P\otimes\P)(\b x,\b y)  \\ 
    &\overset{\text{(e)}}{=} \int_{\R^d}\int_{\R^d}\int_{\R^d}  i\fip{\b y e^{-\imag\fip{\b y, \w}{2}}, \w}{\C^d}e^{\imag\fip{\b x, \w}{2}} \d\P(\b x) \d\P(\b y) \d\Lambda(\w) \\
    &\overset{\text{(f)}}{=} \int_{\R^d}\int_{\R^d} i\fip{\b y e^{-\imag\fip{\b y, \w}{2}}, \w}{\C^d} \int_{\R^d}  e^{\imag\fip{\b x, \w}{2}} \d\P(\b x) \d\P(\b y) \d\Lambda(\w) \\
    &\overset{\text{(g)}}{=} \int_{\R^d} i\bigfip{\underbrace{\int_{\R^d}\b y e^{-\imag\fip{\b y, \w}{2}}\d\P(\b y)}_{\overset{\text{Lemma~\ref{lemma_aux:derivative_of_char}(ii)}}{=}\imag\nabla_{\w}\psi_\P(-\w)}, \w}{\C^d}\psi_\P(\w)   \d\Lambda(\w) 
    \overset{\text{(h)}}{=} \int_{\R^d} -\fip{\nabla_{\w}\psi_{\P}( -\w),\w}{\C^d} \psi_{\P}(\w)\d\Lambda(\w), \hspace{.7cm} \label{eq:deriv-t3}
\end{align}
where (a) follows by Bochner's theorem (recalled in Theorem~\ref{thm:bochner}), (b) is shown in Lemma~\ref{lemma_aux:derivative_of_kernel}(i). The integration is swapped with the inner product in (c). (d) follows from the conjugate-linearity of the complex inner product in the second argument, $\co{e^{iz}} = e^{-iz}$ ($z\in \R$) and $e^{\imag\fip{\b x - \b y, \w}{2}} = e^{\imag\fip{\b x, \w}{2}} e^{-\imag\fip{\b y, \w}{2}}$. Fubini's theorem, validated in Lemma~\ref{lemma_aux:fubini}(iii) by using \eqref{eq:finite_moments}, the product structure of $\P\otimes \P$, and the linearity of the complex inner product in its first argument yield (e). The linearity of the integration gives (f).  By the definition of characteristic function, the linearity of integration, and exchanging the inner product and the integral, we obtain (g). Lemma~\ref{lemma_aux:derivative_of_char}(ii) [validated in \eqref{eq:charact_is_c2}], the linearity of the complex inner product in the first argument, and $\imag^2 = -1$ yield (h). 

\tb{Term $t_4$.} Last, we rewrite $t_4$ as
 \begin{align}
    t_4 &=\int_{\R^d\times \R^d}\sum_{j=1}^{d}\frac{\partial^2k(\b x, \b y)}{\partial x_j \partial y_j} \d (\P\otimes\P)(\b x, \b y)
    \overset{\text{(a)}}{=} \int_{\R^d\times\R^d} \sum_{j=1}^d \int_{\R^d} {\w^{2\b e_j}} e^{ -\imag\fip{\b x - \b y,\w}{2}}\d\Lambda(\w)\d (\P\otimes\P)(\b x, \b y) \\
    &\overset{\text{(b)}}{=} \int_{\R^d\times\R^d} \int_{\R^d} \norm{\w}{2}^2e^{-i\fip{\b x - \b y, \w}{2}} \d \Lambda(\w) \d (\P\otimes \P)(\b x, \b y) \\
    &\overset{\text{(c)}}{=} \int_{\R^d}\int_{\R^d\times\R^d} \norm{\w}{2}^2e^{-i\fip{\b x - \b y, \w}{2}} \d (\P\otimes \P)(\b x, \b y) \d \Lambda(\w) 
    \overset{\text{(d)}}{=}  \int_{\R^d} \norm{\w}{2}^2\psi_\P(-\w) \psi_\P(\w) \d \Lambda(\w). \label{eq:deriv-t4}
\end{align}
Lemma~\ref{lemma_aux:derivative_of_kernel}(iii) gives (a), while linearity of the integral and observing that $\sum_{j=1}^d\w^{2e_j} = \sum_{j=1}^d\omega_j^2 = \norm{\w}{2}^2$ yields (b). (c) follows by applying Fubini's theorem, verified in Lemma~\ref{lemma_aux:fubini}(iv) by using \eqref{eq:finite_moments}. The product structure of $\P \otimes \P$, the property $e^{-i\fip{\b x - \b y, \w}{2}} = e^{-i\fip{\b x , \w}{2}} e^{i\fip{\b y, \w}{2}}$, the linearity of the integration, and the definition of the characteristic function imply (d).

\subsection{Proof of Lemma~\ref{lemma:lsksd-closed-form-p-gaussian}} \label{sec:proof-lsksd-closed-form-p-gaussian}
As $\P=\mathcal N(\bm \mu,\bm \Sigma)$, it holds that $M_{\bm \alpha}^{ {\P}} < \infty$ for all $\bm \alpha \in \N^d$. Hence, by Lemma~\ref{lemma:lsksd-closed-form}, we obtain that 
\begin{align}
    \Ds\P =  \int_{\R^d}& \norm{\nabla_{\w} \psi_P(\w) + \w\psi_P(\w)}{\C^d}^2  \d \Lambda(\w).
\end{align}

Recall that the characteristic function of a multivariate normal is $\psi_{\P}(\w) = e^{i\fip{\bm \mu,\w}{2} - \frac{1}{2}\fip{\w,\bm \Sigma\w}{2}}$. Thus, $\nabla_{\bm\w}\psi_{\P}(\w) = (i\bm \mu  - \bm \Sigma \bm \w) e^{i\fip{\bm \mu,\w}{2} - \frac{1}{2}\fip{\w,\bm \Sigma\w}{2}} = \b z\,\psi_\P(\w)$, with $\b z\coloneq i\bm\mu-\bm\Sigma\w$.
To obtain the stated expression, we rewrite the integrand as\begin{align}
    \norm{\nabla_{\w} \psi_P(\w) + \w\psi_P(\w)}{\C^d}^2
    &\overset{}{=}\norm{\b z \psi_P(\w) + \w\psi_P(\w)}{\C^d}^2
    \overset{\text{(a)}}{=} \|\b z +\w\|^2_{\C^d}\norm{\psi_{\P}(\w)}{\C}^2
    \\
    &\overset{\text{(b)}}{=} \|i\bm\mu-\bm\Sigma\w +\w\|^2_{\C^d}\, \norm{\psi_{\P}(\w)}{\C}^2
    \overset{\text{(c)}}{=} \big( \norm{\bm\mu}{2}^{2} + \norm{\w - \bm \Sigma\w}{2}^{2}\big)\norm{\psi_{\P}(\w)}{\C}^2,
\end{align}
In (a) we used the homogeneity of norms, (b) follows by the definition of $\b z$, and using that $\|\b z\|^2_{\C^d} = \|\Re(\b z)\|^2_2 + \|\Im(\b z)\|_2^2$ yields (c).

\subsection{Proof of Theorem~\ref{thm:minimax}} \label{sec:proof-minimax}
Fix $j\in[d]$, $n \in \Zp$, and denote by $\mathcal G = \big\{\mathcal N(\rho\, \b e_j,\b I_d):~\rho \geq 0\big \} \subset \mathcal M_1^+(\R^d)$ a subset of the Gaussian measures on $\R^d$. As this family is parameterized by $\rho \ge 0$, we write $G_\rho \in \mathcal{G}$.\footnote{Since $k$ is bounded, all $f \in \mathcal{H}_k$ are bounded. Then, we have that $\lim_{ \|\b x\|_{2} \to \infty}g_0(\b x)f(\b x) = 0,$ with $g_0$ the density of $G_0$ w.r.t.\ the Lebesgue measure, implying that $G_0\in\mathcal{T}$.} We proceed by lower bounding the l.h.s.\ of \eqref{eq:main-result-langevin} and then applying Theorem~\ref{theorem:le-cam-main-text}. In particular, for any $C > 0$, we have 
\begin{align} 
\MoveEqLeft\inf_{\hat F_n}\sup_{\P_0\in\mathcal T} \sup_{\P\in\mathcal{S}_{\P_0}}  \P^n\!\Big(\underbrace{\left|\D{\P}-\hat F_n\right|}_{=\hat \Delta_n} > C\Big) 
\overset{\text{(a)}}{\ge}  \inf_{\hat F_n}\sup_{\P_0 \in \{G_0\}} \sup_{\P\in\mathcal{S}_{\P_0}}  \P^n\!\left(\left|\D{\P}-\hat F_n\right|> C\right) \\ 
&\overset{\text{(b)}}{=}   \inf_{\hat F_n}\sup_{\P\in\mathcal{S}_{G_0}}  \P^n\!\left(\left|\Dq\P-\hat F_n\right|>C\right)
\overset{\text{(c)}}{\ge}   \inf_{\hat F_n}\sup_{G\in\mathcal G} G^n\!\left(\left|\Dq{G}-\hat F_n\right| >C\right), \label{eq:minimax-reduction}
\end{align}
where we obtain (a) as $\mathcal T \supseteq \{G_0\}$ and (b) by noting that the supremum of a singleton is attained at its element. To prove the inclusion $\mathcal{S}_{G_0}\supseteq\mathcal{G}$ used in (c), we observe that for any $G \in \mathcal G$, we have
\begin{equation}
    \E_G\sqrt{K_0(X,X)} \overset{\text{(a)}}{=} \E_G\norm{K_0(\cdot,X)}{\mathcal H_{K_0}} \overset{\text{(b)}}{\le} \left( \E_G\norm{K_0(\cdot,X)}{\mathcal H_{K_0}}^2\right)^{1/2} \overset{\text{(a)}}{=} \left( \E_G K_0(X,X) \right)^{1/2}.
\end{equation}
(a) holds by the fact that in a Hilbert space the norm is induced by the inner product and by using the reproducing property, (b) is implied by Jensen's inequality. The final term satisfies the bound
\begin{align}
        \E_G K_0(X,X) &\overset{\text{(c)}}{=} \int_{\mathbb{R}^d} \norm{\nabla_{\b{x}} \log p_0(\b x)}{2}^2 \kappa(0)\d G(\b x) 
        \overset{\text{(d)}}{=}  \int_{\mathbb{R}^d} \norm{\b x}{2}^2 \kappa(0)\d G(\b x) \overset{\text{(e)}}{<} \infty,
\end{align}
where the definition of $K_0$ implies (c), as $k(\b x, \b y) = \kappa(\b x - \b y) = \kappa(0)$ is constant and thus its derivatives are zero. In (d), we recall (from the proof of Lemma~\ref{lemma:lsksd-closed-form} in Appendix~\ref{sec:proof-lsksd-closed-form}) that $\nabla_{\b x} \log p_0(\b x) = -\b x$ as $\P_0$ has (Lebesgue) density $p_0(\b x) \propto e^{-\norm{\b x}{2}^2/2}$ by assumption. Noticing that Gaussian $G$-s have finite second moments gives (e) and proves that $\E_G\sqrt{K_0(X,X)} < \infty$; hence, $G \in \mathcal S_{G_0}$, which was to be shown.

To bring ourselves into the setting of Theorem~\ref{theorem:le-cam-main-text}, 
we let $\Y = \left(\R^d\right)^n$, $\Theta =\left\{\theta_\rho = \Dq{G_\rho}\,:\,\rho \ge 0 \right\}$, $d(x,y) = |x-y|$ ($x,y\in\R)$, and $\mathcal{P}_\Theta = \left\{G^n_{\rho}\,:\, \rho \ge 0\right\} = \left\{G^n_{\rho}\,:\,G_{\rho} \in \mathcal G\right\} = \left\{G^n_{\rho}\,:\,\theta_\rho \in \Theta\right\}$ therein.
Hence, the observed data $X_{1:n} \in \mathcal{Y}$ is distributed as $X_{1:n}\sim G_{\rho}^n\in \mathcal{P}_\Theta$ for some unknown $\theta_\rho \in \Theta$. 
Let $\hat{F}_n = \hat{F}_n(X_{1:n})$ be any estimator of $\Dq{G_\rho}$  based on the $n$ samples~$X_{1:n}$.

In this setting, we consider the adversarial pair $(\theta_\rho,\theta_0) = \big(\Dq{G_\rho},\Dq{G_0}\big) = \big(\Dq{G_\rho},0\big)$ with our choice of $\rho = 1/\sqrt{n}$; it remains to lower bound $d(\theta_\rho,\theta_0)$ and to upper bound $\KL(G^n_\rho\|G^n_0)$.

\begin{enumerate}[label=(\roman*)]
\item \tb{Lower bound for $d(\theta_\rho,\theta_0)$.} We obtain for the squared distance that
\begin{align}
    d^2(\theta_\rho,\theta_0) 
    &\overset{\text{(a)}}{=} \Dqs{G_\rho}
    \overset{\text{(b)}}{=}\rho^{2}\int_{\R^d}\norm{\psi_{G_\rho}(\w)}{\C}^2\d\Lambda(\w)
    \overset{\text{(c)}}{=} \rho^{2}\int_{\R^d}e^{-\norm{\w}{2}^2}\d\Lambda(\w)
    \overset{\text{(d)}}{\geq} \rho^{2}\int_{A} e^{-\norm{\w}{2}^2}\d\Lambda(\w)\\
    &\overset{\text{(e)}}{\geq} \rho^{2} \Lambda(A)\inf_{\bm \w \in A}e^{-\norm{\w}{2}^2} 
    \overset{\text{(f)}}{=} \rho^{2} \Lambda(A)e^{-\delta_0}
    \overset{\text{(g)}}{\geq} \rho^{2}\underbrace{ \Lambda(B)e^{-\delta_0}}_{\eqcolon 4c^2}
    \overset{\text{(h)}}{=}  \frac{4c^2}{n},    \label{eq:lower-bound-i}
\end{align}
where our choice of $(\theta_\rho,\theta_0)$ gives (a). (b) holds by Lemma~\ref{lemma:lsksd-closed-form-p-gaussian}, and (c) follows by recalling that $\psi_{\mathcal N(\bm \mu,\bm \Sigma)}(\w) = e^{\imag\fip{\bm \mu,\w}{2}-\frac12\fip{\w,\bm \Sigma\w}{2}}$ implies that $\norm{\psi_{G_\rho}(\w)}{\C}^2=\psi_{G_\rho}(\w)\co{\psi_{G_\rho}(\w)}=e^{-\norm{\w}{2}^2}$.  We define the closed ball with fixed radius $0 < \delta_0 < \infty$, $A=\{\w \in \R^d\,:\, \|\w\|_2^2\leq \delta_0\}\subset \R^d$, which is compact, and use the positivity of the exponential function with the monotonicity of the integral in (d). 
Considering the infimum of the integrand with the monotonicity of the integration, and the integration of constant functions gives (e). In (f), we use that a continuous function on a compact domain attains its infimum and the definition of $A$. Let $B \subset A$ be the interior of $A$; we then use the monotonicity of measures to obtain (g). Since $k$ is characteristic, $\supp(\Lambda) = \mathbb{R}^d$ (Theorem~\ref{thm:bharath-full-Rd}), implying that $\Lambda(B) > 0$ (as the interior $B$ is open), ensuring that $c>0$. (h) follows from our choice of $\rho=1/\sqrt n$. Finally, taking the square root of \eqref{eq:lower-bound-i}, we have 
\begin{equation}
    d(\theta_\rho,\theta_0) \geq \frac{2c}{\sqrt{n}} \eqcolon 2s > 0. \label{eq:2s-definition}
\end{equation}

\item \tb{Upper bound for $\KL(G^n_\rho\|G^n_0)$.} We have the chain of equalities
\begin{align}
    \KL(G^n_\rho\|G^n_0) &\overset{\text{(a)}}{=} \sum_{j=1}^{n} \KL(G_\rho\|G_0)
   \overset{\text{(b)}}{=} \frac{n}{2}\left(d + \rho^2\norm{\b e_j}{2}^2 - d + \ln(1)\right)
   \overset{\text{(c)}}{=} \frac{1}{2},
\end{align}
where (a) holds by Lemma~\ref{lemma:tsy-kl-prod-measure} and (b) by Lemma~\ref{lemma:kl-gaussians}. In (c), we use our choice of $\rho$. Hence, letting $\alpha \coloneq \frac{1}{2}$, we have
\begin{equation}
\operatorname{KL}(G^n_\rho\|G^n_0) \leq \alpha = \frac{1}{2}.\label{eq:kl-1/2-bound}
\end{equation}

\end{enumerate}

Then, by invoking Theorem~\ref{theorem:le-cam-main-text}, we obtain for \eqref{eq:minimax-reduction} using $C = s = c/\sqrt n$, with $s$ defined in \eqref{eq:2s-definition}, that
\begin{equation}
    \inf_{\hat F_n}\sup_{G \in\mathcal G} G^n\!\left(\left|\Dq{G}-\hat F_n\right| >\frac{c}{\sqrt n}\right) \ge \max\left(\frac{e^{-1/2}}{4}, \frac{1-\sqrt{1/4}}{2}\right) = \frac{1}{4},\label{eq:solution_thm_1}
\end{equation}
concluding the proof.

\subsection{Proof of Corollary~\ref{cor:minimax-k-gaus}} \label{sec:proof-cor-minimax-gauss}

By the proof of Theorem~\ref{thm:minimax} (Appendix~\ref{sec:proof-minimax}), in particular \eqref{eq:lower-bound-i} and \eqref{eq:2s-definition}, it is sufficient to make the dependence of $s_n$ on our choice of $k(\b x, \b y) = e^{-\gamma\norm{\b x-\b y}{2}^2}$ ($\b x, \b y \in \R^d$) explicit. 
We proceed in two steps:
\begin{enumerate}[label=(\roman*)]
    \item First, we obtain a closed-form expression for $\dfrac{\d \Lambda}{\d \lambda_d}$, with $\Lambda$ corresponding to the spectral measure associated to the Gaussian kernel.
    \item Second, we also obtain $d(\theta_\rho,\theta_0)$ in closed form, using the density obtained in (i), which will imply the stated result.
\end{enumerate}

The details are as follows.

\begin{enumerate}[label=(\roman*)]
    \item \tb{Closed-form of $\d \Lambda / \d \lambda_d$.} Recall that by Bochner's theorem (Theorem~\ref{thm:bochner}), 
        \begin{equation}
        k(\b x, \b y) = \kappa(\b x - \b y) = \int_{\R^d}e^{-\imag\fip{\b x - \b y,\w}{2}}\d \Lambda (\w) = \int_{\R^d}\cos\big(\fip{\b x - \b y,\w}{2}\big)\d \Lambda (\w), \label{eq:kappa-bochner}
        \end{equation}
        where the last equation is implied by Euler's formula ($e^{ix} = \cos(x) + i\sin(x)$ for $x \in \R$), the definition of the complex integral, and as $k$ is real-valued. By \citet[(4) and Table~2]{sriperumbudur10hilbert} $\kappa$ has Fourier transform $\mathcal F\kappa$ given by (with $\gamma=1/\big(2\sigma^2\big)$ therein)
        \begin{equation}
            (\mathcal F\kappa)(\w) = \frac{1}{(2\pi)^{d/2}} \int_{\R^d}e^{-i\fip{\b z,\w}{2}}\kappa(\b z) \d \b z = \sigma^de^{-\frac{\sigma^2\norm{\w}{2}^2}{2}}. \label{eq:explicit-fourier-trans}
        \end{equation}
        Using this expression, the Fourier inversion theorem now implies that
        \begin{align}
            k(\b x, \b y) &= \kappa(\b x - \b y) = \mathcal F^{-1} (\mathcal F\kappa)(\b x-\b y) =  \frac{1}{(2\pi)^{d/2}} \int_{\R^d}e^{i\fip{\b x - \b y,\w}{2}}(\mathcal F\kappa)(\w) \d \w \\
            &= \frac{1}{(2\pi)^{d/2}} \int_{\R^d}\cos\big(\fip{\b x - \b y,\w}{2}\big)(\mathcal F\kappa)(\w) \d \w, \label{eq:inverse-fourier}
        \end{align}
        where Euler's formula, $\kappa$ and $\mathcal F\kappa$ being real-valued, and the definition of the complex integral imply the last expression.
        
        As \eqref{eq:kappa-bochner} and \eqref{eq:inverse-fourier} are equal, we obtain that
        \begin{equation}
            \dfrac{\d \Lambda}{\d \lambda_d }(\w) = \frac{1}{(2\pi)^{d/2}}(\mathcal F\kappa)(\w) \overset{\text{(a)}}{=} \frac{\sigma^d}{(2\pi)^{d/2}}e^{-\frac{\sigma^2\norm{\w}{2}^2}{2}} \overset{\text{(b)}}{=} \frac{1}{(4\pi\gamma)^{d/2}}e^{-\frac{\norm{\w}{2}^2}{4\gamma}}, \label{eq:closed-form-Gaussian-kernel-density}
        \end{equation}
        by using the explicit form of $\mathcal F\kappa$ [\eqref{eq:explicit-fourier-trans}] in (a) and $\gamma=1/\big(2\sigma^2\big)$ in (b).
    \item \tb{Closed-form of $d(\theta_\rho,\theta_0)$.} From \eqref{eq:lower-bound-i}(c), we have
    \begin{align}
        d^2(\theta_\rho,\theta_0) = \rho^2 \int_{\R^d}e^{-\norm{\w}{2}^2}\d\Lambda(\w) \overset{\text{(a)}}{=} \rho^2c_1 \int_{\R^d}e^{-\norm{\w}{2}^2}e^{-\frac{\norm{\w}{2}^2}{4\gamma}}\d\w 
        \overset{\text{(b)}}{=} \rho^2c_1 \int_{\R^d}e^{-c_2\norm{\w}{2}^2}\d\w, \label{eq:intermediate-closed-form}
    \end{align}
     with (a) following from  \eqref{eq:closed-form-Gaussian-kernel-density} and letting $c_1 \coloneq 1/(4\pi\gamma)^{d/2}$, and in (b) setting $c_2 \coloneq 1 + \frac{1}{4\gamma}$. Recall that the Gaussian integral has closed-form solution $\int_\R e^{-ax^2}\d x = (\pi/a)^{1/2}$ for $a>0$; hence
    \begin{equation}
        \int_{\R^d}e^{-c_2\norm{\w}{2}^2}\d\w = \prod_{j=1}^d\int_\R e^{-c_2 \omega_j^2}\d \omega_j = \prod_{j=1}^d \left(\frac{\pi}{c_2}\right)^{1/2} = \left(\frac{\pi}{c_2}\right)^{d/2},
    \end{equation}
    which, continuing from \eqref{eq:intermediate-closed-form}, gives
    \begin{equation}
        d^2(\theta_\rho,\theta_0) = \rho^2c_1\left(\frac{\pi}{c_2}\right)^{d/2} \overset{\text{(a)}}{=} \rho^2 \left(\frac{1}{4\pi\gamma}\frac{\pi}{1+\frac{1}{4\gamma}}\right)^{d/2} \overset{\text{(b)}}{=} \rho^2\left(\frac{1}{4\gamma+1}\right)^{d/2},
    \end{equation}
    using our definitions of $c_1$ and $c_2$ in (a) and simplifying in (b).

    Our choice of $\rho = 1/\sqrt n$ and taking the positive square root yields that 
    \begin{equation}
        d(\theta_\rho,\theta_0) = \frac{1}{\sqrt n}\left(\frac{1}{4\gamma+1}\right)^{d/4} \eqcolon 2 s. \label{eq:sn-def-corollary}
    \end{equation}
\end{enumerate}
Following the notation in \eqref{eq:2s-definition}, one gets that $c\coloneq (4\gamma + 1)^{-d/4}/2$.

\subsection{Proof of Theorem~\ref{thm:minimax-lower}}
\label{sec:proof-minimax-lower}
Observe that, for a $P'_0$ defined as in Assumption~\ref{ass:gen_ksd}, we have
\begin{align}
    \inf_{\hat{F}_n} \sup_{P_0 \in \mathcal T}\sup_{P\in \mathcal{S}_{P_0}} \P^n\Big(\underbrace{\left|\D{\P}-\hat F_n\right|}_{=\hat \Delta_n}\geq C\Big)
    &\overset{\text{(a)}}{\geq} \inf_{\hat{F}_n} \sup_{\P_0\in \{P'_0\}}\sup_{P\in \mathcal{S}_{P_0}} \P^n\Big(\left|\D{\P}-\hat F_n\right|\geq C\Big)
    \\&\overset{\text{(b)}}{=}  \inf_{\hat{F}_n} \sup_{P\in \mathcal{S}_{P'_0}} \P^n\Big(\left|\KSD(P'_0,P)-\hat F_n\right|\geq C\Big), \label{eq:thm2-lhs-bound}
\end{align}
where (a) comes by the fact that $\{P'_0\}\subset \mathcal T$ and (b) by noting that the supremum of a singleton is attained at its element.  In the following, we relabel $P_0'$ as $P_0$; in other words, we write $P_0 = P'_0$.

To bring ourselves into the setting of Theorem~\ref{theorem:le-cam-main-text}, for any fixed $n\in \Zp$, set $\Y \coloneq (\R^d)^n$, $\Theta \coloneq \{\theta_P \coloneq \KSD(P_0,P)\,:\, P \in \mathcal S_{P_0}\}$, $\mP_{\Theta} \coloneq \{P^n\,:\, P\in  \mathcal S_{P_0}\} = \{P^n\,:\, \theta_P \in \Theta\}$, and $d(x,y) \coloneq |x-y|$ ($x,y\in \R$). Let us define $F: \mathcal S_{P_0} \to \R$ by $P\mapsto \KSD(P_0,P)$, and let $\hat{F}_n$ denote the corresponding estimator based on $n$ samples. We construct $(P_{\theta_{0}(n)},P_{\theta_{1}(n)})$ for fixed $n$, where $P_{\theta_{0}(n)} \coloneq P_{\theta_0}$ with $\theta_0 \coloneq \theta_{P_0}$, $P_{\theta_1(n)} \coloneq P_{\theta_n}$ with $\theta_n \coloneq \theta_{P_n}$ and $P_n$ specified below in \eqref{eq:defq1}. With these notations at hand, $d(\theta_0(n),\theta_1(n))=|\KSD(P_0,P_0)-\KSD(P_0,P_n)| = |0- \KSD(P_0,P_n)| = \KSD(P_0,P_n)$. 

Next, we present the \tb{construction of the adversarial sequence $P_n$}. 
Let $\varphi\in \mathcal C_b(\X)$ be as constructed in Lemma~\ref{lem:phi-exists}, that is, (i) satisfying $\E_{P_0}[\varphi(X)] = 0$ and (ii) guaranteeing that there exists $A'\in\mathcal B(\X)$ with positive $P_0$-measure such that $\varphi(x) \neq 0$ for all $x\in A'$. We construct $P_n$ as a  perturbation of $P_0$ taking the form
\begin{align}
P_n(A) = \int_A1 + \epsilon_n\varphi(x)\d P_0(x) \text{ for any } A \in \mathcal B(\X),
\label{eq:defq1}
\end{align}
with $\epsilon_n = c n^{-1/2}$, where the precise value of $c>0$ will be specified later; we also note that $P_n\neq P_0$ by Lemma~\ref{lem:p_n-neq-p_0}. \eqref{eq:defq1} implies that $P_n \ll P_0$ and the corresponding Radon-Nikodym derivative takes the form 
\begin{equation}
    \dfrac{\d P_n}{\d P_0} = 1 + \epsilon_n\varphi. \label{eq:pn-radon-nikodym}
\end{equation}

We show that $P_n \in \mathcal S_{P_0}$  for sufficiently large $n$.
Indeed:
\begin{enumerate}
\item $P_n\ge 0$ for $n\ge n_{0,1}$: Recalling from \eqref{eq:defq1} that for $A\in\mathcal B(\X)$ 
\begin{align}
    P_n(A) = \int_A 1+\epsilon_n\varphi(x)\d P_0(x),
\end{align}
it suffices to show that $1+\epsilon_n\varphi(x)\ge 0$ for all $x\in \X$ and $n$ large enough.
As $\varphi \in \mathcal C_b(\X)$, $\varphi$ is bounded and 
\begin{equation}
L \coloneq \inf_{x \in \X} \varphi(x) > -\infty. \label{eq:L-def}
\end{equation}
Further, by the construction of $\varphi$, $\E_{P_0}[\varphi(X)] = 0$; hence
\begin{equation}
    0 = \E_{P_0}[\varphi(X)] = \int_\X \varphi(x)\d P_0(x) \overset{\text{(a)}}{\ge}  \int_\X \inf_{x\in\X}\varphi(x) \d P_0(x) \overset{\text{(b)}}{=} LP_0(\X) \overset{\text{(c)}}{=} L;
\end{equation}
in other words, $L\le 0$. (a) holds by the monotonicity of the integration, (b) follows from the definition of $L$ and the integration of constants, (c) comes from $P_0\in  \mathcal M_1^+(\X)$.

For any $x\in\X$, it holds that $1+\epsilon_n\varphi(x) \ge \inf_{x\in\X} [1+\epsilon_n\varphi(x)] = 1+ \epsilon_nL$, and we are done once we establish that the last term is non-negative: 

\begin{equation}
    1 + \epsilon_n L \ge 0 \iff 1 -\epsilon_n |L| \ge 0 \iff 1 \ge \epsilon_n |L| \iff \frac{1}{\epsilon_n} \ge |L|,
\end{equation}
where we used that the non-positivity of $L$ means that $L = -|L|$.
By using that $\epsilon_n = c n^{-1/2}$ with $c>0$, we have that $1/\epsilon_n = n^{1/2}/c \to \infty$ as $n\to \infty$, guaranteeing $1/\epsilon_n \ge |L|$ for $n$ large enough (say, $n\ge n_{0,1}$).

\item $P_n(\X) = 1$:  One has
\begin{align}
    P_n(\X) \overset{\text{(a)}}{=} \int_\X 1+\epsilon_n\varphi(x)\d P_0(x) \overset{\text{(b)}}{=} 1 + \epsilon_n\underbrace{\int_{\X}\varphi(x)\d P_0(x)}_{\overset{\text{(c)}}{=}0} = 1.
\end{align}
(a) follows from the definition of $P_n$ [\eqref{eq:defq1}]; (b) is by the linearity of integration and using that $\int_\X 1\d P_0(x) = P_0(\X) = 1$; (c) uses the mean-zero property of $\varphi$ w.r.t.\ $P_0$.

\item $\E_{P_n} \sqrt{K_0(X,X)} < \infty$: One gets
\begin{align}
    \MoveEqLeft\E_{P_n}\sqrt{K_0(X,X)} \overset{\text{(a)}}{=} \int_\X \sqrt{K_0(x,x)} [1+\epsilon_n\varphi(x)]\d P_0(x)\\
    &\overset{\text{(b)}}{=}\underbrace{\int_\X \sqrt{K_0(x,x)} \d P_0(x)}_{\eqcolon t_1} + \epsilon_n\underbrace{\int_\X \sqrt{K_0(x,x)} \varphi(x)\d P_0(x)}_{\eqcolon t_2}.
\end{align}
The first step (a) is by the definition of the expectation and by the properties of the Radon-Nikodym derivative [\eqref{eq:pn-radon-nikodym}]. In (b), we use the linearity of the integral. Term $t_1$ is finite by applying \eqref{eq:finiteness-of-KSD} with $P = P_0$. For $t_2$, let $\sup_{x\in\X}|\varphi(x)| \eqcolon M < \infty$, where the finiteness of $M$ holds by $\varphi \in \mathcal C_b(\X)$. We have
\begin{equation}
    \left|\int_\X \sqrt{K_0(x,x)} \varphi(x)\d P_0(x)\right| \overset{\text{(a)}}{\le} \int_\X \sqrt{K_0(x,x)} |\varphi(x)|\d P_0(x) \overset{\text{(b)}}{\le} M \int_\X \sqrt{K_0(x,x)}\d P_0(x) \overset{\text{(c)}}{=} Mt_1 \overset{\text{(d)}}{<} \infty,
\end{equation}
by applying in (a) Jensen's inequality and using the non-negativity of $\sqrt{K_0(x,x)}$ ($x\in\X$), in (b) the definition of $M$ with the monotonicity and linearity of the integration, in (c) the definition of $t_1$, in (d) the finiteness of $M$ and $t_1$.
\end{enumerate}

Having defined $P_n$, we continue with the \tb{control of the KSD value $\KSD\left(P_0,P_n\right)$}:
\begin{align}
    \KSD(P_0,P_n) &\overset{\eqref{eq:gen-ksd-norm-repr}}{=} \big\|\E_{P_n}[\Psi_{P_0}(X)]\big\|_{\H} \overset{\text{(a)}}{=} \Big\|\E_{P_0}\Big[\Psi_{P_0}(X)\big(1 + \epsilon_n\varphi(X)\big)\Big]\Big\|_\H\\
    &\,\overset{\text{(b)}}{=} \big\|\underbrace{\E_{P_0}[\Psi_{P_0}(X)}_{=0 \impliedby \eqref{eq:gen-ksd-mean-zero}}] + \epsilon_n\E_{P_0}[\varphi(X)\Psi_{P_0}(X)]\big\|_\H \stackrel{\text{(c)}}{=} \epsilon_n\underbrace{\big\|\E_{P_0}[\varphi(X)\Psi_{P_0}(X)]\big\|_\H}_{\eqcolon C_{\varphi}} \overset{\text{(d)}}{>} 0,
\end{align}
where in (a) we used the definition of $P_n$ and the property of the Radon-Nikodym derivative, (b) holds by the linearity of the expectation, (c) is implied by the homogeneity of norms and the positivity of $\epsilon_n$,  and (d) follows from the fact that $\epsilon_n > 0$ and that by $P_n \neq P_0$ we have $\KSD(P_0,P_n) > 0$ by the validity of $\KSD$ imposed in Assumption~\ref{ass:gen_ksd}. Hence,
\begin{align}
\KSD\left(P_0,P_n\right) & = \epsilon_n C_\varphi \stackrel{\text{(a)}}{=} \Theta\!\left(n^{-1/2}\right), \label{eq:KSD-bound}
\end{align}
where (a) holds by $\epsilon_n = c n^{-1/2}$ ($c>0$) and  $C_\varphi > 0$.

We proceed by \tb{controlling the KL divergence $\KL\left(P_{n} \| P_0\right)$}:
\begin{align}
\MoveEqLeft \KL\left(P_{n} \| P_0\right) \stackrel{\text{(a)}}{=} \E_{P_n}\ln\left[\dfrac{\d P_n}{\d P_0}(X)\right] 
\stackrel{\text{(b)}}{=}
\E_{P_0}\big[(1+\epsilon_n\varphi(X))\ln(1+\epsilon_n\varphi(X))\big], \label{eq:KL:integral-form}
\end{align}
where in (a) the definition of the KL divergence was applied, (b) is implied by the definition of $P_n$ [\eqref{eq:pn-radon-nikodym}] and the properties of the Radon-Nikodym derivative.

To gain control over the integral in \eqref{eq:KL:integral-form}, we recall that, for any $x>-1$, one has that $\ln(1+x)\le x$. Let $n$ be large enough (say $n\ge n_{0,2})$ such that for all $x \in \X$ one has $|\epsilon_n \varphi(x)| < 1$; this is possible as $\varphi$ is bounded. Then, we can upper bound \eqref{eq:KL:integral-form} as
\begin{align}
    \E_{P_0}\big[\underbrace{(1 + \epsilon_n\varphi(X))}_{> 0} \underbrace{\ln(1 + \epsilon_n\varphi(X))}_{\le \epsilon_n\varphi(X)} \big] 
    &\overset{\text{(a)}}{\le} \E_{P_0}\big[(1 + \epsilon_n\varphi(X))\epsilon_n\varphi(X) \big] 
    \overset{\text{(b)}}{=}\epsilon_n\underbrace{\E_{P_0}\big[\varphi(X)\big]}_{=0} +\epsilon_n^2\underbrace{\E_{P_0}\big[\varphi^2(X)\big]}_{\eqcolon M_{2} < \infty} \\
    &= M_{2} \epsilon_n^2 \overset{\text{(c)}}{=} \O(1/n). \label{eq:KL-order}
\end{align}
In (a), we use the monotonicity and in (b) the linearity of  integration. The function $\varphi$ has zero-mean w.r.t.\ $P_0$ by construction; it is also bounded, guaranteeing the finiteness of $M_{2}$. Our choice of $\epsilon_n = cn^{-1/2}$ yields (c) and we \tb{choose} $c$ in the following.

Indeed, from \eqref{eq:KL-order} and the definition of $\epsilon_n$, one gets that
\begin{align}
n \KL(P_n \| P_0) &\le nM_{2} \frac{c^2}{n} = M_{2}c^2;
\end{align}
hence, by choosing $c \coloneq \sqrt{\ln (2)}/\sqrt{M_{2}} > 0$, we arrive at
\begin{align}
n \KL(P_n \| P_0) & \le \ln (2).
\end{align}
Thus, for sufficiently large $n$ (say $n\ge n_{0,2}$), the requirement $n \KL(P_n \| P_0)\le \ln (2)=:\alpha$ in Theorem~\ref{theorem:le-cam-main-text} is fulfilled, and $n \ge n_0\coloneq \max(n_{0,1},n_{0,2})$ incorporates all our $n$ is large enough constraints. With our choice of $c$, by the definition of $\epsilon_n$, \eqref{eq:KSD-bound} translates to
\begin{align}
\KSD\!\left(P_0,P_n\right) & = n^{-1/2} c C_\varphi =: 2s_n,
\end{align}
defining $s_n\coloneq \frac{n^{-1/2} c C_\varphi}{2}$; $s_n>0$ since $c C_\varphi>0$ by $c>0$ and $C_{\varphi} > 0$. Hence, Theorem~\ref{theorem:le-cam-main-text} together with \eqref{eq:thm2-lhs-bound} implies that for all $n\ge n_0$ 
\begin{align}
 \inf_{\hat{F}_n} \sup_{P_0 \in \mathcal T}\sup_{P\in \mathcal{S}_{P_0}} P^n\big(\hat\Delta_n\geq s_n\big) \ge f(\alpha),
\end{align}
with $f$ defined in Theorem~\ref{theorem:le-cam-main-text}.\footnote{$f(\ln(2)) = \max\left(\frac{1}{8},\frac{1-\sqrt{\frac{\ln(2)}{2}}}{2}\right)=\frac{1-\sqrt{\frac{\ln(2)}{2}}}{2} \approx 0.29$.} This means that for all $n\ge n_0$
\begin{align}
\inf_{\hat{F}_n} \sup_{P_0 \in \mathcal T}\sup_{P\in \mathcal{S}_{P_0}}P^n\Big(\hat\Delta_n\geq n^{-1/2} \underbrace{\frac{ c  C_\varphi}{2}}_{\eqcolon B>0}\Big) \ge f(\alpha), 
\end{align}
which concludes the proof.

\section{AUXILIARY RESULTS}
In this section, we collect a few auxiliary results. Lemma~\ref{lemma_aux:fubini} validates our applications of Fubini's theorem in the proof of Theorem~\ref{thm:minimax}. Lemma~\ref{lemma_aux:derivative_of_char} relates the gradient of a distribution's characteristic function to its moments. Lemma~\ref{lemma_aux:derivative_of_kernel} is about the derivatives of a continuous bounded translation-invariant kernel in terms of its Bochner representation.
Lemma~\ref{lem:phi-exists} shows the existence of a bounded smooth perturbation function.

\begin{lemmaA}[Lebesgue integrability of key functions]\label{lemma_aux:fubini}
    Let $\P\in \mathcal{M}_1^+ (\R^d)$, $\Lambda$ a finite non-negative measure on $\big(\R^d,\mathcal B(\R^d)\big)$, and $\Lambda' = \frac{\Lambda}{\Lambda(\R^d)}$.\footnote{This normalization implies that $\Lambda' \in \mathcal{M}_1^+\!\left(\R^d\right)$.} Assume that for all $|\bm \alpha| \le 2$ with $\bm \alpha \in \N^d$, $M_{\bm \alpha}^\P< \infty$ and $M_{\bm \alpha}^{\Lambda'}< \infty$. 
    Then,
    \begin{enumerate}[label=(\roman*)]
        \item $\int_{\R^d\times \R^d\times \R^d} \left|\fip{\b x,\b y}{\C^d} e^{-\imag\fip{\b x - \b y, \w}{2}}\right| \d (\Lambda\otimes\P\otimes\P)(\w, \b x, \b y) < \infty $,
        \item $\int_{\R^d\times \R^d\times \R^d} \left|\fip{\b x,\w}{\C^d} e^{\imag\fip{\b x - \b y, \w}{2}}\right| \d (\Lambda\otimes\P\otimes\P)(\w, \b x, \b y) < \infty$,
        \item $\int_{\R^d\times \R^d\times \R^d} \left|\fip{\b y,\w}{\C^d} e^{\imag\fip{\b x - \b y, \w}{2}}\right| \d (\Lambda\otimes\P\otimes\P)(\w, \b x, \b y) < \infty$,
        \item $\int_{\R^d\times \R^d\times \R^d} \left|\fip{\w,\w}{\C^d} e^{-\imag\fip{\b x - \b y, \w}{2}}\right| \d (\Lambda\otimes\P\otimes\P)(\w, \b x, \b y) < \infty$.
    \end{enumerate}
\end{lemmaA}
\begin{proof}
We prove the finiteness of each integral separately.  %

\tb{Integral (i).} One has
\begin{align}
    \MoveEqLeft\int_{\R^d\times \R^d\times \R^d} \left|\fip{\b x,\b y}{\C^d} e^{-\imag\fip{\b x - \b y, \w}{2}}\right| \d (\Lambda\otimes\P\otimes\P)(\w, \b x, \b y)
    \overset{\text{(a)}}{ =} \Lambda(\R^d)\int_{\R^d\times \R^d\times \R^d} \left|\fip{\b x,\b y}{\C^d} \right| \d (\Lambda'\otimes\P\otimes\P)(\w, \b x, \b y) \\
    &\overset{\text{(b)}}{ \le} \Lambda(\R^d) \Bigg[\int_{\R^d\times \R^d\times \R^d} \left|\fip{\b x,\b y}{\C^d} \right|^2 \d (\Lambda'\otimes\P\otimes\P)(\w, \b x, \b y)\Bigg]^{1/2} \\
    &\overset{\text{(c)}}{=} \Lambda\!\left(\R^d\right)  \Bigg[\int_{\R^d\times \R^d} \left|\fip{\b x,\b y}{\C^d} \right|^2 \d (\P\otimes\P)(\b x, \b y)\Bigg]^{1/2} \\
    &\overset{\text{(d)}}{\le} \Lambda\!\left(\R^d\right)  \Bigg[\int_{\R^d\times \R^d} \norm{\b x}{2}^2\norm{\b y}{2}^2 \d (\P\otimes\P)(\b x, \b y)\Bigg]^{1/2} 
    \overset{\text{(e)}}{=}  \Lambda\!\left(\R^d\right)  \Bigg(\int_{\R^d} \norm{\b x}{2}^2\d \P(\b x)\Bigg)^{1/2}\Bigg(\int_{\R^d}\norm{\b y}{2}^2 \d \P(\b y)\Bigg)^{1/2}\\
    &\overset{\text{(f)}}{=} \Lambda\!\left(\R^d\right) \left(\sum_{j=1}^d M^{\P}_{2\b e_j}\right)^{1/2} \left(\sum_{j=1}^d M^{\P}_{2\b e_j}\right)^{1/2}
    \overset{\text{(g)}}{<} \infty, \label{eq:fubini-t1}
\end{align}
where (a) follows by noting that $\left| e^{-i \fip{\b x - \b y, \w}{2}}\right| = 1$ and the definition of $\Lambda'$. 
The monotonicity of $L_p$ norms w.r.t.\ $p$ with probability measures yields (b).
In (c), we use the product structure of $\Lambda'\otimes \P\otimes \P$, and that $\Lambda'\!\left(\R^d\right) = 1$ as $\Lambda'  { \in \mathcal{M}_1^+}\left(\R^d\right)$. 
To obtain (d), we apply the CBS inequality and that $\|\b x\|_{\C^d} = \|\b x\|_{2}$ and $\|\b y\|_{\C^d} = \|\b y\|_{2}$ when $\b x, \b y \in \R^d$, (e) is by independence. (f) comes from the definition of $\|\cdot\|_2$, the linearity of integration, and the definition of $M^{\P}_{\bm \alpha}$. (g) follows by observing that Bochner's theorem guarantees the finiteness of $\Lambda\!\left(\R^d\right)$ and since $M^{\P}_{\bm \alpha} < \infty$ for $|\bm \alpha| \leq 2$ by assumption. 

\tb{Integral (ii).} 
Observe that
\begin{align}
    \MoveEqLeft \int_{\R^d\times \R^d \times \R^d} \left|\fip{\b x, \w}{\C^d}e^{\imag\fip{\b x - \b y, \w}{2}}  \right| {\d (\Lambda\otimes\P\otimes\P)(\w, \b x,\b y)}\overset{\text{(a)}}{ =} \Lambda(\R^d)\int_{\R^d\times \R^d\times \R^d} \left|\fip{\b x,\w}{\C^d} \right| \d (\Lambda'\otimes\P\otimes\P)(\w, \b x, \b y) \\
    &\hspace{-.6cm}\overset{\text{(b)}}{\leq} \Lambda(\R^d)\Bigg[\int_{\R^d\times \R^d \times \R^d} \left|\fip{\b x, \w}{\C^d} \right|^{2} {\d (\Lambda'\otimes\P\otimes\P)(\w, \b x,\b y)} \Bigg]^{1/2}
    \overset{\text{(c)}}{ =} \Lambda(\R^d)\Bigg[\int_{\R^d\times \R^d} \left|\fip{\b x, \w}{2} \right|^{2} {\d (\Lambda'\otimes\P)(\w, \b x)} \Bigg]^{1/2}
    \\
    &\hspace{-.6cm} \overset{\text{(d)}}{\le} {\Lambda(\R^d)} \left( \int_{\R^d}\norm{\b x}{2}^2 \d \P(\b x)\right)^{1/2} \left( \int_{\R^d}\norm{\w}{2}^2 \d \Lambda'(\w)\right)^{1/2}
    \overset{\text{(e)}}{=}\Lambda(\R^d) \left( \sum_{j=1}^d M_{2\b e_j}^\P \right)^{1/2}\left(\sum_{j=1}^d M_{2\b e_j}^{\Lambda'} \right)^{1/2} 
    \overset{\text{(f)}}{<} \infty, \label{eq:fubini-t2}
\end{align}
where (a) comes by noting that $|e^{i \fip{\b x - \b y, \w}{2} }| =   1$ and the definition of $\Lambda'$, and (b) by applying the monotonicity of $L_p$ norms as in part (i). Noticing that $\P(\R^d)=1$ and that $\fip{\b x, \w}{\C^d} = \fip{\b x, \w}{2}$ for real vectors yields (c).
To get (d), we apply the CBS inequality and independence. (e) follows from the definition of $\|\cdot\|_2$, the linearity of the integral, and by the definition of $M_{\bm \alpha}^\P$.  
To obtain (f),  note that (i) $\Lambda(\R^d) < \infty$ by Bochner's theorem, and (ii) $M_{2\b e_j}^{\Lambda'} < \infty$ and $M_{2\b e_j}^\P < \infty$ for all $j \in [d]$ by assumption.

\tb{Integral (iii).} We have 
\begin{align}
    \MoveEqLeft \int_{\R^d\times \R^d \times \R^d} \left|\fip{\b y, \w}{\C^d}e^{\imag\fip{\b x - \b y, \w}{2}}  \right| {\d (\Lambda\otimes\P\otimes\P)(\w, \b x,\b y)}\\
    &\overset{\text{(a)}}{ =} \Lambda(\R^d)\int_{\R^d\times \R^d\times \R^d} \left|\fip{\b y,\w}{\C^d} \right| \d (\Lambda'\otimes\P\otimes\P)(\w, \b x, \b y) \overset{\text{(b)}}{ < } \infty
\end{align}
where (a) comes from $| \fip{\b y, \w}{\C^d}e^{i \fip{\b x - \b y, \w}{2} }| = |\fip{\b y, \w}{\C^d}|$ and the definition of $\Lambda'$. 
With a change of the variables $\b y$ and $\b x$, \eqref{eq:fubini-t2} yields (b).

\tb{Integral (iv).} We obtain bounds for the last integral by noting that $\fip{\w,\w}{\C^d}=\norm{\w}{\C}^2=\norm{\w}{2}^2$ for $\w \in \R^d$ and considering
\begin{align}
    \MoveEqLeft \int_{ {\R^d\times\R^d\times\R^d}} \left|\norm{\w}{2}^2e^{-i\fip{\b x - \b y, \w}{2}}\right| \d (\Lambda\otimes\P\otimes\P)(\w,\b x, \b y)
    \overset{\text{(a)}}{=} \int_{\R^d\times\R^d\times\R^d} \norm{\w}{2}^2\d (\Lambda\otimes  \P\otimes \P)(\w,\b x, \b y)\\
    & \overset{\text{(b)}}{=}  {\int_{\R^d} \norm{\w}{2}^2 \d \Lambda(\w)}
    \overset{\text{(c)}}{=}  {\Lambda(\R^d)\int_{\R^d} \norm{\w}{2}^2 \d \Lambda'(\w)} 
    \overset{\text{(d)}}{=}  {\Lambda(\R^d)\int_{\R^d} \sum_{j=1}^d \w^{2\b e_j} \d \Lambda'(\w)}
    \overset{\text{(e)}}{=} \Lambda(\R^d) \sum_{j=1}^d M_{2\b e_j}^{\Lambda'} \overset{\text{(f)}}{<}\infty, \label{eq:fubini-t4}
\end{align}
where (a) uses that $|e^{\imag z}| = 1$ for any $z\in \R$. (b) follows from the product structure of $\Lambda\otimes\P\otimes \P$ and the property $\P(\R^d) =1$. Our definition of $\Lambda = \Lambda(\R^d)\Lambda'$ gives (c) and we make the definition of $\norm{\cdot}{2}^2$ explicit in (d). We swap the integral with the sum by using the linearity of the integration in (e) and use the notation for moments. (f) is implied by the assumed finiteness of $M_{2\b e_j}^{\Lambda'}$ for all $j\in [d]$.
\end{proof}

\begin{lemmaA}[Gradient of characteristic function]\label{lemma_aux:derivative_of_char} 
    Let $\Q \in \mathcal{M}_1^+(\R^d)$ with characteristic function $\psi_\Q$. If $D^{\b e_j}\psi_\Q$ exists for all $j\in [d]$, then for all $\w \in \R^d$, one has
    \begin{enumerate}[label = (\roman*)]
        \item $\nabla_{\w} \psi_\Q(\w) = i\int_{\R^d} \b x e^{i \fip{\b x, \w}{2}}\d\Q(\b x)$, and
        \item $\nabla_{\w} \psi_\Q(-\w) = -i\int_{\R^d} \b x e^{-i \fip{\b x, \w}{2}}\d\Q(\b x)$. 
    \end{enumerate}
\end{lemmaA}
\begin{proof}
    Observing that $D^{\b e_j}\psi_\Q (\w) = i \int_{\R^d} \b x^{\b e_j} e^{i\fip{\b x, \w}{2}}\d \Q(\b x)$ by Theorem~\ref{thm:diff-char-func} and that the expectation of a vector is the vector of expectations yield (i). We obtain (ii) by writing
    \begin{align}
        \nabla_{\w}\psi_\Q (-\w) \overset{\text{(a)}}{=} \nabla_{\w}\overline{\psi_\Q (\w)} \overset{\text{(b)}}{=} \overline{\nabla_{\w}\psi_\Q (\w)} \overset{\text{(c)}}{=}  -i\int_{\R^d} \b x e^{-i \fip{\b x, \w}{2}}\d\Q(\b x),
    \end{align}
    where (a) comes by the definition of the characteristic function, (b) follows from the fact that the derivative of the conjugate is the conjugate of the derivative, and (c) is implied by taking the conjugate of the result obtained in (i).
\end{proof}

\begin{lemmaA}[Derivatives of the kernel via its Bochner's representation]\label{lemma_aux:derivative_of_kernel}
    Let $k$ be a kernel satisfying Assumption~\ref{ass:kernel_k_LS} and $k\in \mathcal{C}^2(\R^d\times\R^d)$ with Bochner representation $k(\b x, \b y) = \int_{\R^d} e^{-i\fip{\b x - \b y, \w}{2}} \d \Lambda(\w).$ Then, 
    \begin{enumerate}[label=(\roman*)]
        \item $\nabla_{\b x} k(\b x, \b y) = - i \int_{\R^d} \w e^{-i\fip{\b x - \b y,\w}{2}}\d \Lambda(\w)$, 
        
        \item $\nabla_{\b y} k(\b x, \b y) =   i \int_{\R^d} \w e^{-i\fip{\b x - \b y,\w}{2}}\d \Lambda(\w)$, 
        
        \item $\frac{\partial}{\partial \b x^{\b e_j} \partial \b y^{\b e_j}} k(\b x, \b y) =    \int_{\R^d} \w^{2\b e_j} e^{-i\fip{\b x - \b y,\w}{2}}\d \Lambda(\w)$. 
    \end{enumerate}
\end{lemmaA}

\begin{proof}
    Throughout the proof, let $\Lambda' = \frac{\Lambda}{\Lambda(\R^d)}$, where we note that $\Lambda' \in \mathcal{M}_1^+(\R^d)$. Furthermore, let $ {g(\b x, \b y)} = \b y - \b x$.
    We show each statement separately. 

    \tb{Part (i).} Considering the Bochner representation of $k(\b x, \b y)$ allows us to write 
    \begin{align}
    \nabla_{\b x} k(\b x, \b y) &=\nabla_{\b x} \int_{\R^d} e^{-i\fip{\b x - \b y, \w}{2}}\d \Lambda(\w) 
    \overset{\text{(a)}}{=} \begin{pmatrix}
                    \Lambda(\R^d)\frac{\partial}{\partial \b x^{\b e_1}}\int_{\R^d}e^{\imag\fip{ {g(\b x, \b y)}, \w}{2}}\d \Lambda'(\w) \\
                    \vdots \\
                    \Lambda(\R^d)\frac{\partial}{\partial \b x^{\b e_d}}\int_{\R^d}e^{\imag\fip{ {g(\b x, \b y)} \w}{2}}\d \Lambda'(\w) 
                 \end{pmatrix} 
    \overset{\text{(b)}}{=} \begin{pmatrix}
                    \Lambda(\R^d)\frac{\partial}{\partial \b x^{\b e_1}}\psi_{\Lambda'}( {g(\b x, \b y)}) \\
                    \vdots \\
                    \Lambda(\R^d)\frac{\partial}{\partial \b x^{\b e_d}}\psi_{\Lambda'}( {g(\b x, \b y)}) 
                 \end{pmatrix} \\
    &\overset{\text{(c)}}{=} \begin{pmatrix}
                    \Lambda(\R^d)\frac{\partial  {g(\b x, \b y)}}{\partial \b x^{\b e_1}}\left.D^{\b e_1}\psi_{\Lambda'}(\b t)\right|_{\b t = \b y - \b x} \\
                    \vdots \\
                    \Lambda(\R^d)\frac{\partial  {g(\b x, \b y)}}{\partial \b x^{\b e_d}}\left.D^{\b e_d}\psi_{\Lambda'}(\b t)\right|_{\b t = \b y - \b x}
                 \end{pmatrix} 
   \overset{\text{(d)}}{=} \begin{pmatrix}
                    \Lambda(\R^d) (-\imag^{|\b e_1|})\int_{\R^d}\w^{\b e_1}e^{\imag\fip{\b y -  \b x, \w}{2}}\d \Lambda'(\w) \\
                    \vdots \\
                    \Lambda(\R^d) (-\imag^{|\b e_d|})\int_{\R^d}\w^{\b e_d}e^{\imag\fip{\b y -  \b x, \w}{2}}\d \Lambda'(\w)
                 \end{pmatrix}\\
   &\overset{\text{(e)}}{=} -\imag\int_{\R^d} \w e^{-\imag\fip{\b x -  \b y, \w}{2}}\d \Lambda(\w), \label{eq:diff_t3}
\end{align}
where (a) comes by the definitions of $\nabla_{\b x}$, $\Lambda'$, $g$, and the linearity of the inner product. (b) stems from the definition of the characteristic function and (c) follows from the chain rule. Theorem~\ref{thm:diff-char-func} and the substitution $\b t = \b y - \b x$ yield (d). Last, we recall that the expectation of a random vector equals the vector of the expectations of its components, which, together with the definition of $\Lambda'$ and the linearity of the inner product, imply (e). 

    \tb{Part (ii).} Observing that $\frac{\partial g(\b x, \b y)}{\partial \b y^{\b e_j}} = 1$, we can write \begin{align}
    \nabla_{\b y} k(\b x, \b y) &= \nabla_{\b y} \int_{\R^d} e^{-i\fip{\b x - \b y, \w}{2}}\d \Lambda(\w) 
    \overset{\text{(a)}}{=} \begin{pmatrix}
                    \Lambda(\R^d)\frac{\partial}{\partial \b y^{\b e_1}}\int_{\R^d}e^{\imag\fip{ {g(\b x, \b y)}, \w}{2}}\d \Lambda'(\w) \\
                    \vdots \\
                    \Lambda(\R^d)\frac{\partial}{\partial \b y^{\b e_d}}\int_{\R^d}e^{\imag\fip{ {g(\b x, \b y)} \w}{2}}\d \Lambda'(\w) 
                 \end{pmatrix} 
    \overset{\text{(b)}}{=} \begin{pmatrix}
                    \Lambda(\R^d)\frac{\partial}{\partial \b y^{\b e_1}}\psi_{\Lambda'}( {g(\b x, \b y)}) \\
                    \vdots \\
                    \Lambda(\R^d)\frac{\partial}{\partial \b y^{\b e_d}}\psi_{\Lambda'}( {g(\b x, \b y)}) 
                 \end{pmatrix} \\
    &\overset{\text{(c)}}{=} \begin{pmatrix}
                    \Lambda(\R^d)\frac{\partial  {g(\b x, \b y)}}{\partial \b y^{\b e_1}}\left.D^{\b e_1}\psi_{\Lambda'}(\b t)\right|_{\b t = \b y - \b x} \\
                    \vdots \\
                    \Lambda(\R^d)\frac{\partial  {g(\b x, \b y)}}{\partial \b y^{\b e_d}}\left.D^{\b e_d}\psi_{\Lambda'}(\b t)\right|_{\b t = \b y - \b x}
                 \end{pmatrix} 
   \overset{\text{(d)}}{=} \begin{pmatrix}
                    \Lambda(\R^d)\imag^{|\b e_1|}\int_{\R^d}\w^{\b e_1}e^{\imag\fip{\b y -  \b x, \w}{2}}\d \Lambda'(\w) \\
                    \vdots \\
                    \Lambda(\R^d)\imag^{|\b e_d|}\int_{\R^d}\w^{\b e_d}e^{\imag\fip{\b y -  \b x, \w}{2}}\d \Lambda'(\w)
                 \end{pmatrix}\\
   &\overset{\text{(e)}}{=} \imag\int_{\R^d} \w e^{-\imag\fip{\b x -  \b y, \w}{2}}\d \Lambda(\w), \label{eq:diff_t2}
    \end{align}
    where (a), (b), (c), (d), and (e) were obtained as in part (i).    

    \tb{Part (iii).} Consider the Bochner representation of $k(\b x, \b y)$. Then, 
    \begin{align} 
    \frac{\partial}{\partial \b x^{\b e_j} \partial \b y^{\b e_j}} k(\b x, \b y) & = \frac{\partial^2}{\partial \b x^{e_j} \partial \b y^{e_j} } \int_{\R^d}e^{-\imag\fip{\b x - \b y,\w}{2}}\d\Lambda(\w)
        \overset{\text{(a)}}{=} \Lambda(\R^d) \frac{\partial^2}{\partial \b x^{e_j}  \partial \b y^{e_j} } \int_{\R^d}e^{\imag\fip{\b y - \b x,\w}{2}}\d\Lambda'(\w) \\
        & \overset{\text{(b)}}{=} \Lambda(\R^d) \frac{\partial^2\psi_{\Lambda'}(g(\b x, \b y))}{\partial \b x^{e_j}  \partial \b y^{e_j} }
        \overset{\text{(c)}}{=} \Lambda(\R^d)\underbrace{\frac{\partial g(\b x,\b y)}{\partial  {\b x^{e_j} }}}_{= -1}\underbrace{\frac{\partial g(\b x, \b y)}{\partial  {\b y^{e_j} }}}_{= 1}\left.D^{2\b e_j}\psi_{\Lambda'}(\b t)\right|_{\b t=\b y - \b x}\\
        &\overset{\text{(d)}}{= } - i^2 \int_{\R^d} \w^{2\b e_j} e^{-i\fip{\b x - \b y,\w}{2}}\d \Lambda(\w) , \label{eq:t4-auxiliary}
    \end{align}
    where (a) comes by $\Lambda = \Lambda(\R^d)\Lambda'$, the linearity of the integral, the partial derivative, and the inner product. The definitions of $g$ and characteristic function yield (b). The chain rule gives (c) and Theorem~\ref{thm:diff-char-func} with $\bm \alpha = 2\b e_j$ implies (d). Noting that $i^2=-1$ leads to the claimed result.
\end{proof}

\begin{lemmaA}[Existence of perturbation function] \label{lem:phi-exists} Let $(\X,\tau_\X)$ be a topological space, $P_0 \in \mathcal M_1^+(\X)$, and $\varphi_0\in\mathcal C_b(\X)$ such that there exists no $c\in\R$ such that $\varphi_0=c$ holds $P_0$-almost surely. Then there exists $\varphi \in  \mathcal C_{b}(\X)$ such that $\E_{P_0}[\varphi(X)] = 0$ and there exists a set $A\in\mathcal B(\X)$ with positive $P_0$-measure such that $\varphi(x) \neq 0$ for all $x\in A$. 
\end{lemmaA}
\begin{proof}
 Since $\varphi_0 \in \mathcal C_b(\X)$, $\varphi_0$ is integrable w.r.t.\ $P_0$ and $\mu_0 \coloneq \E_{P_0}[\varphi_0(X)] <\infty$. By centering $\varphi_0$ as $\varphi \coloneq \varphi_0 - \mu_0$, 
one has $\E_{P_0}[\varphi(X)] = \E_{P_0}[\varphi_0(X)] - \mu_0 = \mu_0 - \mu_0 = 0$. Also, as $\varphi_0$ is not constant $P_0$-almost surely, the property of  $\varphi= 0$ $P_0$-almost surely does not hold, implying the existence of the stated $A\in\mathcal B(\X)$. To see that $\varphi \in \mathcal C_b(\X)$, it suffices to note that (i) $\varphi$ is the sum of continuous functions and thus continuous, and (ii) $\sup_{x \in \X}|\varphi(x)| \le \sup_{x \in \X}|\varphi_0(x)| + |\mu_0| < \infty$ by the triangle inequality, hence $\varphi$ is also bounded.
\end{proof}

\begin{lemmaA}[Perturbed measures are distinct]\label{lem:p_n-neq-p_0}
    Assume $P_0\in\mathcal M_1^+(\X)$, let $\varphi\in\mathcal C_b(\X)$ be such that there exists $A'\in\mathcal B(\X)$ with positive $P_0$-measure, such that $\varphi(x) \neq 0$ for all $x\in A'$, and define the measure $P_n$ as $P_n(A) = \int_A1+\epsilon_n\varphi(x)\d P_0(x)$, with $\epsilon_n>0$, for all $A\in \mathcal B(\X)$. Then, $P_0 \neq P_n$.
\end{lemmaA}
\begin{proof}
    We argue by contradiction. Assume that $P_0 = P_n$. Then, \begin{align}
        P_0=P_n 
        &\overset{\text{(a)}}{\implies} 1=\frac{\d P_n}{\d P_0} = 1+\epsilon_n \varphi \text{ $P_0$-almost surely} 
        \overset{\text{(b)}}{\implies} \varphi = 0 \text{ $P_0$-almost surely},
    \end{align}
    where (a) uses the definition of the Radon-Nikodym derivative and (b) follows as $\epsilon_n>0$.
    This contradicts the assumption imposed on $\varphi$, concluding the proof.
\end{proof}

\section{EXTERNAL STATEMENTS}
To ensure self-completeness, this section collects the external statements that we use. Theorem~\ref{thm:bochner} fully characterizes continuous bounded translation-invariant kernels. Theorem~\ref{thm:diff-char-func} relates the differentiability of the characteristic function of a random variable to its moments; we include only the part relevant to our proofs for brevity. Under certain conditions, the converse also holds, detailed in Theorem~\ref{thm:exist-of-moments}. Theorem~\ref{thm:bharath-full-Rd} gives a necessary and sufficient condition for a continuous bounded translation-invariant kernel to be characteristic. 
We recall Fubini's theorem in Theorem~\ref{th::ex::Fubini}. Lemma~\ref{lemma:tsy-kl-prod-measure} and Lemma~\ref{lemma:kl-gaussians} collect properties of the KL divergence.

\begin{theoremA}[Bochner; Theorem~6.6; \citealt{wendland05scattered}] \label{thm:bochner}
    A continuous function $\kappa : \R^d \to \R$ is positive definite if and only if it is the Fourier transform of a finite nonnegative Borel measure $\Lambda$ on $\R^d$, that is,
    \begin{align*} 
     \kappa(\b x) = \int_{\R^d}e^{-\imag\fip{\b x,\w}{2}}\d\Lambda(\w)\quad \text{for all } \b x \in \R^d.
    \end{align*}
\end{theoremA}

\begin{theoremA}[Differentiability characteristic function; Theorem 1.2.1(i); \citealt{sasvari13charfunc}] \label{thm:diff-char-func}
Let $X\sim\P \in \mathcal M_1^+\left(\R^d\right)$ and $\bm \alpha \in \N^d$ such that the moment $M_{\bm \alpha}^{\P}$ of $\P$ exists. Then the partial derivative $D^{\bm \alpha}\psi_\P$ exists and one has  $D^{\bm \alpha}\psi_\P(\b t) = \imag^{|\bm \alpha|} \int_{\R^d}\b x^{\bm\alpha} e^{i\fip{\b t,\b x}{2}}\d \P(\b x)$ ($\b t \in \mathbb R^d$).
\end{theoremA}

\begin{theoremA}[Existence of the moments of $P$; Theorem 1.2.9; \citealt{sasvari13charfunc}] \label{thm:exist-of-moments}
    Let $X\sim\P \in \mathcal M_1^+\left(\R^d\right)$ and $\bm \alpha \in \N^d\setminus\{\bm{0}\}$ such that all partial derivatives $D^{\bm \beta} \Re(\psi_{P})(\b t),~\bm \beta < 2\bm \alpha$ exist in an open neighborhood of zero. If $D^{2\bm \alpha}\Re(\psi_{P})(\b t)$ exists at zero, then the moment $M_{2\bm\alpha}^\P$ of $\P$ exists.
\end{theoremA}

\begin{theoremA}[Theorem 9; \citealt{sriperumbudur10hilbert}] \label{thm:bharath-full-Rd} Suppose $k:\R^d \times \R^d \to \R$ is a  continuous bounded translation-invariant kernel. Then $k$ is characteristic if and only if $\operatorname{supp}(\Lambda) = \R^d$, with $\Lambda$ defined according to  Theorem~\ref{thm:bochner} as $k(\b x,\b y) = \int_{\R^d}e^{-\imag\fip{\b x - \b y,\w}{2}}\d\Lambda(\w)$ ($\b x, \b y \in \R^d$).
\end{theoremA}

The following theorem allows to exchange the order of integration. We recall that $\sigma$-finiteness always holds for any Borel probability measure.

\begin{theoremA}[Fubini-Tonelli; Theorem 2.37.b; \citealt{folland99real}]\label{th::ex::Fubini}
    Suppose that $(\mathcal{X},\mathcal{M},\mu)$ and $(\mathcal{Y},\mathcal{N},\nu)$ are $\sigma$-finite measure spaces. Let $f:  \mathcal{X}\times \mathcal{Y} \to \R$.
    If $\int_{\X\times\mathcal Y} |f(x,y)| \d (\mu\otimes\nu)(x,y) < \infty$, then 
    \begin{equation}
        \int_{\mathcal{X}\times \mathcal{Y}}f(x,y)\d(\mu\otimes\nu)(x, y) = \int_\mathcal{X} \left[\int_\mathcal{Y} f(x,y) \d \nu(y) \right]\d \mu (x) = \int_\mathcal{Y} \left[\int_\mathcal{X} f(x,y) \d \mu(x) \right]\d \nu (y).
    \end{equation}
     
\end{theoremA}

\begin{lemmaA}[KL divergence of product measures; p.~85; \citealt{tsybakov09introduction}]
\label{lemma:tsy-kl-prod-measure}
Let $\P=\otimes_{j=1}^n\P_j$ and $\Q=\otimes_{j=1}^n\Q_j$. Then
\begin{align*}
    \mathrm{KL}(\P||\Q) = \sum_{j=1}^n\mathrm{KL}(\P_j||\Q_j).
\end{align*}
\end{lemmaA}

\begin{lemmaA}[KL divergence of Gaussians; p.~13; \citealt{duchi07derivationscorrectlink}]
\label{lemma:kl-gaussians}
The KL divergence of two normal distributions $\mathcal{N}(\bm
\mu_1,\bm\Sigma_1)$ and $\mathcal{N}(\bm \mu_0,\bm\Sigma_0)$ on $\R^d$   is
\begin{align}
  \mathrm{KL}(\mathcal{N}(\bm \mu_1,\bm\Sigma_1)||\mathcal{N}(\bm \mu_0,\bm\Sigma_0)) = \frac{
    \trace(\bm\Sigma_0^{-1} \bm\Sigma_1) + (\bm\mu_0-\bm\mu_1)^{\T}\bm\Sigma_0^{-1}(\bm\mu_0-\bm\mu_1) - d +\ln\left(\frac{\left|\bm\Sigma_0\right|}{\left|\bm\Sigma_1\right|}\right)
    }{2}.
\end{align}
\end{lemmaA}

\bibliography{BIB/collected_Florian, BIB/extra_jose, BIB/publications_Florian, BIB/collected_plus}

\end{document}